\documentclass[11pt]{article}

\usepackage{jmlr2e}
\firstpageno{1}

\usepackage{microtype}
\usepackage{graphicx}
\usepackage{subfigure}
\usepackage{booktabs} 
\usepackage{multirow}
\usepackage{amssymb}
\usepackage{mathtools}

\usepackage{enumitem}

\usepackage{amsmath}
\usepackage{algorithm}
\usepackage{algorithmic}
\newtheorem{assumption}{Assumption}

\def \x {\mathbf{x}}

\def \x {\mathbf{x}}

\def \F {\mathcal{F}}

\usepackage{hyperref}
\usepackage[capitalize,noabbrev]{cleveref}
\hypersetup{colorlinks={true},linkcolor={blue},citecolor={blue}}

\begin{document}

\title{Stochastic Continuous Submodular Maximization: Boosting via Non-oblivious Function}

\author{\name Qixin Zhang \email qxzhang4-c@my.cityu.edu.hk\\
       \addr School of Data Science\\
       City University of Hong Kong\\
       Kowloon, Hong Kong, China
      \AND
       \name Zengde Deng \email zengde.dzd@cainiao.com \\
      \addr Cainiao Network\\
      Hang Zhou, China
       \AND
       \name Zaiyi Chen \email zaiyi.czy@cainiao.com \\
      \addr Cainiao Network\\
      Hang Zhou, China
      \AND
       \name Haoyuan Hu \email haoyuan.huhy@cainiao.com \\
      \addr Cainiao Network\\
      Hang Zhou, China
      \AND
      \name Yu Yang \email yuyang@cityu.edu.hk\\
      \addr School of Data Science\\
       City University of Hong Kong\\
       Kowloon, Hong Kong, China
      }

\editor{Anonymous}
\maketitle

% otherwise use the standard text.
\begin{abstract}
In this paper, we revisit Stochastic Continuous Submodular Maximization in both offline and online settings, which can benefit wide applications in machine learning and operations research areas. We present a boosting framework covering gradient ascent and online gradient ascent. The fundamental ingredient of our methods is a novel non-oblivious function $F$ derived from a factor-revealing optimization problem, whose any stationary point provides a  $(1-e^{-\gamma})$-approximation to the global maximum of the $\gamma$-weakly DR-submodular objective function $f\in C^{1,1}_L(\mathcal{X})$. Under the offline scenario, we propose a boosting gradient ascent method achieving $(1-e^{-\gamma}-\epsilon^{2})$-approximation after $O(1/\epsilon^2)$ iterations, which improves the $(\frac{\gamma^2}{1+\gamma^2})$ approximation ratio of the classical gradient ascent algorithm.
In the online setting, for the first time we consider the adversarial delays for stochastic gradient feedback, under which we propose a boosting online gradient algorithm with the same non-oblivious function $F$. Meanwhile, we verify that this boosting online algorithm achieves a regret of $O(\sqrt{D})$ against a $(1-e^{-\gamma})$-approximation to the best feasible solution in hindsight, where $D$ is the sum of delays of gradient feedback. 
To the best of our knowledge, this is the first result to obtain $O(\sqrt{T})$ regret against a $(1-e^{-\gamma})$-approximation with $O(1)$ gradient inquiry at each time step, when no delay exists, i.e., $D=T$. Finally, numerical experiments demonstrate the effectiveness of our boosting methods.
\end{abstract}
% \vspace{-12pt}
\section{Introduction}
Due to the relatively low computational complexity, first-order optimization methods are widely used in machine learning, operations research, and statistics communities. Especially for convex objectives, there is an enormous literature~\citep{nesterov2013introductory,bertsekas2015convex} deriving the convergence rate of first-order methods. Recent studies have shown that first-order optimization methods could also achieve the global minimum for some special non-convex problems~\citep{NIPS2014_443cb001,arora2016computing,ge2016matrix, du2018gradient,liu2020nonconvex}, although it is in general NP-hard to find the global minima of a non-convex objective function~\citep{murty1987some}. Motivated by this, some recent work focused on the structures and conditions under which non-convex optimization is tractable~\citep{bian2017guaranteed,hazan2016graduated}. In this paper, we investigate the stochastic $\gamma$-weakly continuous submodular maximization problem where an unbiased gradient oracle is available under both offline and online scenarios.

Continuous DR-Submodular Maximization has drawn much attention recently due to that it admits efficient approximate maximization routines.
For instance, under the offline deterministic setting, \citet{bian2017guaranteed,bian2020continuous} proposed the vanilla Frank-Wolfe method and its variant achieving $1/2$ and $(1-1/e)$ approximations, respectively. When the stochastic estimates of the gradient is available, \citet{mokhtari2018conditional} and \citet{hassani2020stochastic} proposed some improved variants of the Frank-Wolfe algorithm, equipped with variance reduction techniques. In \citep{hassani2020stochastic}, assuming the Lipschitz continuity of stochastic Hessian, a $[(1-1/e)OPT-\varepsilon]$ solution is achieved using $O(1/\varepsilon^2)$ stochastic gradient. Such a result provides the tightest approximation as well as the optimal stochastic first-order oracle complexity.

However, when generalizing Frank-Wolfe methods to the online setting, some other tricks should be involved, which makes the algorithm design more complicated. For example, \citet{chen2018online} and \citet{zhang2019online} took the idea of meta actions~\citep{streeter2008online} and blocking procedure to propose online Frank-Wolfe algorithms. Moreover, in these aforementioned studies, the environment/adversary reveals the reward and stochastic first-order information immediately after the action is chosen by the learner/algorithm. In practice, the assumption of immediate feedback might be too restrictive. The feedback delays widely exist in many real-world applications, e.g., online advertising~\citep{mehta2007adwords}, influence maximization problem~\citep{chen2012time,yang2016continuous}. 

To tackle these issues, instead of Frank-Wolfe, we adopt Gradient Ascent and aim to propose a uniform algorithmic framework for both offline and online Stochastic Continuous Submodular Maximization. To make the online setting more realistic, we also consider adversarial feedback delays~\citep{quanrud2015online}. Note that our online setting degenerates to the standard online setting if no delay exists. One big challenge in front of us is that the stationary points of a $\gamma$-weakly submodular function $f$ only provide a limited $(\frac{\gamma^{2}}{1+\gamma^{2}})$-approximation to the global maximum~\cite{hassani2017gradient}. As a result, we need to boost stochastic gradient ascent and its online counterpart~\citep{hassani2017gradient} as they only attain a $(\frac{\gamma^{2}}{1+\gamma^{2}})$-approximation. To tackle this challenge, we hope to devise an auxiliary function whose stationary points provide a better approximation guarantee than those of $f$ itself. Motivated by \citep{feldman2011unified,filmus2012power,filmus2014monotone,harshaw2019submodular,feldman2021guess,mitra2021submodular+}, we first consider a family of auxiliary functions whose gradient at point $\boldsymbol{x}$ allocates different weight to the gradient $\nabla f(z*\boldsymbol{x})$ where $z\in[0,1]$. By solving a factor-revealing optimization problem, we select the optimal auxiliary function $F$ whose stationary points provide a tight $(1-e^{-\gamma})$-approximation to the global maximum of $f$. Then, based on this optimal auxiliary function $F$, we propose a simple first-order framework that makes it possible to boost the performance of classical gradient ascent algorithm converging to stationary points.
Based on this boosting framework, we present a boosting gradient ascent and a boosting online gradient ascent to improve the approximation guarantees of vanilla gradient ascent and its online counterpart. To be specific, we make the following contributions:
\begin{enumerate} 
	\item 
	% As we know, many known classical algorithms, such as the projected gradient ascent methods~\citep{hassani2017gradient} and Frank-Wolfe~\citep{lacoste2016convergence}, with small step sizes always converge to a stationary point. 
	% However, the stationary point of a $\gamma$-weakly DR-submodular function $f$ only provides a $(\frac{\gamma^{2}}{1+\gamma^{2}})$-approximation to the global maximum, in contrast with the optimal $(1-e^{-\gamma})$-approximation guarantee. 
	% To boost the performance of these algorithms (converging to stationary points), we consider a family of auxiliary functions related to $f$, whose stationary points provide a better approximation guarantee. 
	We develop a uniform boosting framework, including gradient ascent and online gradient ascent methods. The essential element behind our framework is an optimal auxiliary function $F$ derived from a factor-revealing optimization problem for each $\gamma$-weakly DR-submodular function $f$.
	% For each $\gamma$-weakly DR-submodular function $f$, we derive an optimal auxiliary function $F$ by constructing a factor-revealing optimization problem. 
	The stationary points of $F$ provide a $(1-e^{-\gamma})$-approximation guarantee to the global maximum of $f$. This approximation is better than the $(\frac{\gamma^2}{1+\gamma^2})$-approximation provided by stationary points of $f$ itself.
	% For each function in this family, the gradient at point $\boldsymbol{x}$ allocates different weight $w(z)$ to the gradients $\nabla f(z*\boldsymbol{x})$. 
	% with $w(z)=e^{\gamma(z-1)}$. 
	%Then, we also explore the smoothness, boundness, and unbiased gradient estimate of the optimal $F$.
	
	\item With this non-oblivious function $F$, under the offline setting, we propose the boosting gradient ascent method %for the ($\gamma$-weakly)~DR-submodular maximization problems with an unbiased gradient oracle. Our two methods
	achieving a $(1-e^{-\gamma}-\epsilon^{2})$-approximation after $O(1/\epsilon^{2})$ iterations, which improves the $(\frac{\gamma^2}{1+\gamma^2})$-approximation of the classical projected gradient ascent algorithm and weakens the assumption of high order smoothness on the objective functions \citep{hassani2020stochastic}.
	%\blue{As a result, we give an affirmative answer to question \ref{enum:q1}}. 
	%\item Similarly, we also combine this boosting policy into the classical Frank-Wolfe~\citep{lacoste2016convergence} algorithm, namely, the boosting Frank-Wolfe (\cref{alg:3}), which achieves the $(1-e^{-\gamma}-\epsilon^{2})$-approximation after $O(1/\epsilon^{3})$ iterations. This new algorithm also improve the $(\frac{\gamma^{2}}{1+\gamma^{2}})$-approximation guarantee of the classical Frank-Wolfe~\citep{lacoste2016convergence}.
	
	\item Next, we consider an online submodular maximization setting with adversarial feedback delays. 
	% When the unbiased stochastic gradients estimates are easily acquired, 
	When an unbiased stochastic gradients estimation is available, we propose an online boosting gradient ascent algorithm that theoretically achieves the optimal $(1-e^{-\gamma})$-regret of $O(\sqrt{D})$ with one gradient evaluation for each $f_{t}$, where $D=\sum_{t=1}^{T}d_{t}$ and $d_{t}$ is a positive integer delay for round $t$. 
	% Instead, due to the gradient noise and delays, it is unclear how to theoretically analyze the meta-Frank-Wolfe algorithms~\citep{chen2018online,chen2018projection,zhang2019online} in this new online setting. Similarly, this boosting online algorithm is also better than the traditional online gradient ascent~\citep{zinkevich2003online,hazan2019introduction,chen2018online,quanrud2015online} with the suboptimal $(1/2)$-approximation guarantee. 
	To the best of our knowledge, our work is the first to investigate the adversarial delays in online submodular maximization problems. Moreover, when $D=T$ for the standard no-delay setting, our proposed online boosting gradient ascent algorithm, requiring only $O(1)$ stochastic gradient estimate at each round, yields the first result to achieve $(1-e^{-\gamma})$-approximation with $O(\sqrt{T})$ regret. %Furthermore, when all $d_{t}=1$, we theoretically give a solution to the question \ref{enum:q2}.
	
	\item Finally, we empirically evaluate our proposed boosting methods using the special example of~\citep{hassani2017gradient} and the simulated non-convex/non-concave quadratic programming. Our algorithms have superior performance in the experiments.
\end{enumerate}

\begin{table}[t]
	\renewcommand\arraystretch{1.35}
	\centering
	\caption{\small Comparison of convergence guarantees for continuous DR-submodular function maximization, where the functions are monotone. Except for \cite{bian2017guaranteed} which needs the constraint set to be convex and down-closed, other methods here need the constraint set $\mathcal{C}$ to be convex. Note that 'det.' and 'sto.' represent the deterministic and stochastic setting, respectively. '\textbf{Hess Lip}' means whether the Hessian of functions needs to be Lipschitz continuous, 'OPT' is the function value at the global optimum, '\textbf{Complexity}' is the gradient oracle complexity. For simplicity, we set $\gamma=1$ for our results which reduces to the standard monotone DR-submodular setting.}
	\vspace{0.5em}
	\resizebox{\textwidth}{!}{
		\setlength{\tabcolsep}{1.0mm}{
			\begin{tabular}{c|c|c|c|c}
				\toprule[1.5pt]
				\textbf{Method} & \textbf{Setting}& \textbf{Hess Lip} &\textbf{Utility} & \textbf{Complexity} \\
				\hline
				\hline 
				Submodular FW~\citep{bian2017guaranteed} & det. & No & $(1-1/e)\rm{OPT}-\epsilon$ & $O(1/\epsilon)$ \\
				\hline
				SGA~\citep{hassani2017gradient} & sto. &No & $(1/2)\rm{OPT}-\epsilon$ & $O(1/\epsilon^2)$ \\
				\hline 
				Classical FW~\citep{bian2020continuous} & det. & No & $(1/2)\rm{OPT}-\epsilon$ & $O(1/\epsilon^2)$ \\
				\hline
				SCG~\citep{mokhtari2018conditional} & sto. & No & $(1-1/e)\rm{OPT}-\epsilon$ & $O(1/\epsilon^3)$ \\
				\hline
				SCG++~\citep{hassani2020stochastic} & sto. & Yes & $(1-1/e)\rm{OPT}-\epsilon$ & $O(1/\epsilon^2)$ \\
				\hline
				Non-Oblivious FW~\citep{mitra2021submodular+} & det. & No & $(1-1/e-\epsilon)\rm{OPT}-\epsilon$ & $O(1/\epsilon^3)$ \\
				\hline
				% \hline 
				% Boosting FW~(This paper) & sto. & No & $(1-1/e-\epsilon^2)\rm{OPT}-\epsilon$ & $O(1/\epsilon^3)$ \\
				\hline
				Boosting GA~(This paper) & sto. & No & $(1-1/e-\epsilon^2)\rm{OPT}-\epsilon$ & $O(1/\epsilon^2)$ \\
				\midrule[1.5pt]
			\end{tabular}
	}}
	\vspace{-1.5em}
	\label{tab:offline_convergence}
\end{table}

\subsection{Related Work}
\noindent{\textbf{Submodular Set Functions:}} Submodular set functions originate from combinatorial optimization problems \citep{nemhauser1978analysis,fisher1978analysis,fujishige2005submodular}, which could be either exactly minimized via Lov\'{a}sz extension~\citep{lovasz1983submodular} or approximately maximized via multilinear extension~\citep{chekuri2014submodular}. Submodular set functions find numerous applications in machine learning and other related areas, including viral marketing \citep{kempe2003maximizing}, document summarization \citep{lin2011class}, network monitoring \citep{leskovec2007cost}, and variable selection \citep{das2011submodular,elenberg2018restricted}.

\noindent{\textbf{Continuous Submodular Maximization:}}
Submodularity can be naturally extended to continuous domains. In deterministic setting, \citet{bian2017guaranteed} first proposed a variant of Frank-Wolfe~(Submodular FW) for continuous DR-submodular maximization problem with $(1-1/e)$-approximation guarantee after $O(1/\epsilon)$ iterations. As for the stochastic setting, \citet{hassani2017gradient} proved that the stochastic gradient ascent~(SGA) guarantees a $(1/2)$-approximation after $O(1/\epsilon^{2})$ iterations. Then, \citet{mokhtari2018conditional} proposed the stochastic continuous greedy algorithm~(SCG), which achieves a $(1-1/e)$-approximation after $O(1/\epsilon^{3})$ iterations. Moreover, by assuming the Hessian of objective is Lipschitz continuous, \citet{hassani2020stochastic} proposed the stochastic continuous greedy++ algorithm~(SCG++), which guarantees a $(1-1/e)$-approximation after $O(1/\epsilon^{2})$ iterations.

\noindent{\textbf{Online Continuous Submodular Maximization:}} \citet{chen2018online} first investigated the online (stochastic) gradient ascent~(OGA) with a $(1/2)$-regret of $O(\sqrt{T})$. Then, inspired by the meta actions \citep{streeter2008online}, \citet{chen2018online} also proposed the Meta-Frank-Wolfe algorithm with a $(1-1/e)$-regret bound of $O(\sqrt{T})$ under the deterministic setting. Assuming that an unbiased estimation of the gradient is available, \citet{chen2018projection} proposed a variant of the Meta-Frank-Wolfe algorithm~(Meta-FW-VR) having a $(1-1/e)$-regret bound of $O(T^{1/2})$ and requiring $O(T^{3/2})$ stochastic gradient queries for each function. Then, in order to reduce the number of gradients evaluation, \citet{zhang2019online} presented the Mono-Frank-Wolfe taking the blocking procedure, which achieves a $(1-1/e)$-regret bound of $O(T^{4/5})$ with only one stochastic gradient evaluation in each round. 

\begin{table}[t]
	\renewcommand\arraystretch{1.35}
	\centering
	\caption{Comparison of regrets for stochastic online continuous DR-submodular function maximization with full-information feedback, where the functions are monotone and constraint set $\mathcal{C}$ is convex. Note that '\textbf{\# Grad. Evaluations }' means the number of stochastic gradient evaluations at each round, '\textbf{Ratio}' means approximation ratio, and '\textbf{Delay}' indicates whether the adversarial delayed feedback is considered. For simplicity, we set $\gamma=1$ for our results which reduces to the standard monotone DR-submodular setting, and $D=T$ which means no delay exists.}
	\vspace{0.5em}
	\resizebox{\textwidth}{!}{
		\setlength{\tabcolsep}{1.0mm}{
			\begin{tabular}{c|c|c|c|c}
				\toprule[1.5pt]
				\textbf{Method} & \textbf{\# Grad. Evaluations } &\textbf{Ratio} & \textbf{Regret} &\textbf{Delay} \\
				\hline
				OGA~\citep{chen2018online} & $O(1)$ & $1/2$ & $O(\sqrt{T})$ & No \\
				\hline
				Meta-FW-VR~\citep{chen2018projection} & $T^{3/2}$ & $1-1/e$ & $O(\sqrt{T})$ & No\\
				\hline
				Mono-FW~\citep{zhang2019online}  & $O(1)$ & $1-1/e$ & $O(T^{4/5})$ & No\\
				% \hline 
				% One-Shot FW~\cite{chen2018projection}  & $O(1)$ & $1-1/e$ & $O(T^{2/3})$ & No\\
				\hline
				\hline 
				Boosting OGA~(This paper) & $O(1)$ & $1-1/e$ & $O(\sqrt{T})$ & Yes\\
				\midrule[1.5pt]
			\end{tabular}
	}}
	\vspace{-1.5em}
	\label{tab:online_convergence}
\end{table}

\noindent{\textbf{Non-Oblivious Search:}} In many cases, classical local search, e.g., the greedy method, may return a solution with a poor approximation ratio to the global maximum. To avoid this issue, \citet{khanna1998syntactic} and \citet{alimonti1994new} first proposed a technique named \textsl{Non-Oblivious Search} that leverages an auxiliary function to guide the search. After carefully choosing the auxiliary function, the new solution generated by the non-oblivious search, may have a better performance than the previous solution found by the classical local search. Inspired by this idea, for the maximum coverage problem over a matroid, \citet{filmus2012power} proposed a $(1-1/e)$-approximation algorithm via a non-oblivious set function allocating extra weights to the solutions that cover some element more than once, which efficiently improves the traditional $(1/2)$-approximation greedy method. After that, \citet{filmus2014monotone} extended this idea to improve the $(1/2)$-approximation greedy method for the general submodular set maximization problem over a matroid. Recently, for the continuous submodular maximization problem with concave regularization, a variant of Frank-Wolfe algorithm~(Non-Oblivious FW) based on a special auxiliary function was proposed for boosting the approximation ratio of the submodular part from $1/2$ to $(1-1/e)$ in \citep{mitra2021submodular+}.
Compared to the proposed algorithm in this paper, i) 
The Non-Oblivious Frank-Wolfe method needs $O(1/\epsilon)$ gradient evaluations at each round under the deterministic setting, 
%rather than $O(1)$ evaluations per iteration of our method under the stochastic setting; 
while our method only needs $O(1)$ evaluations per iteration under the stochastic setting;
ii) The Non-Oblivious Frank-Wolfe method is designed only for the deterministic setting, while we present a uniform boosting framework covering the stochastic gradient ascent in both offline and online settings.
% Notably, in contrast with \citet{mitra2021submodular+}, 1) \blue{Their proposed Non-Oblivious Frank-Wolfe method needs $O(1/\epsilon)$ gradient evaluations at each round under the deterministic setting}, but we consider the general stochastic settings and only need $O(1)$ evaluations for every iteration; 2) \blue{Instead of} the Frank-Wolfe, our work presents a uniformly boosting framework covering the (stochastic) gradient ascent in both offline and online settings.

We present comparisons between this work and previous studies in \cref{tab:offline_convergence} and \cref{tab:online_convergence} for offline and online settings, respectively.

% Before going into the detail, we first present some basic concepts and notations in section~\ref{sec:pre}. After that, in section~\ref{sec:non-oblivious}, we demonstrate how to design our auxiliary function via factor-revealing equation, and some properties about the proposed non-oblivious function. Next, we answer the first question \ref{enum:q1} in section~\ref{sec:gradient_ascent}. Namely, the boosting projected gradient ascend (Algorithm~\ref{alg:1}) via non-oblivious function achieves ($1-e^{-\gamma}$)-approximation. Then, in section~\ref{sec:delay}, to demonstrate the power of non-oblivious function, we investigate a complicated online setting with the unknown delays under which the online boosting (stochastic) gradient ascent is proposed with the theoretical guarantee to answer the second question \ref{enum:q2}. Also, in section~\ref{sec:concave}, we demonstrate the relationship between the our non-oblivious function and that in~\citep{mitra2021submodular+}. Finally, we test the boosting methods in some numerical experiments.

\section{Preliminaries}\label{sec:pre}
In this section, we define some concepts and notations that we will use throughout the paper.
\subsection{Continuous Submodularity}
\noindent{\textbf{Continuous Submodular Functions:}} A function $f:\mathcal{X}\rightarrow \mathbb{R}_{+}$ is a {\it continuous submodular} function if for any $ \boldsymbol{x} ,\boldsymbol{y}\in\mathcal{X}$,
\begin{align*}
    f(\boldsymbol{x})+f(\boldsymbol{y})\ge f(\boldsymbol{x}\land\boldsymbol{y})+f(\boldsymbol{x}\lor\boldsymbol{y}).
\end{align*}
Here, $\boldsymbol{x}\land\boldsymbol{y}=\min(\boldsymbol{x},\boldsymbol{y})$ and $\boldsymbol{x}\lor\boldsymbol{y}=\max(\boldsymbol{x},\boldsymbol{y})$ are component-wise minimum and component-wise maximum, respectively. $\mathcal{X}=\prod_{i=1}^{n}\mathcal{X}_{i}$ where each $\mathcal{X}_{i}$ is a compact interval in $\mathbb{R}_{+}$. Without loss of generality, we assume $\mathcal{X}_{i}=[0,a_{i}]$. If $f$ is twice differentiable, the continuous submodularity is equivalent to 
\begin{align*}
    \forall i\neq j,\forall\boldsymbol{x}\in\mathcal{X},\frac{\partial^{2}f(\boldsymbol{x})}{\partial x_{i}\partial x_{j}}\le 0.
\end{align*}
 Moreover, $f$ is {\it monotone} if $f(\boldsymbol{x})\ge f(\boldsymbol{y})$ when $\boldsymbol{x}\ge\boldsymbol{y}$.

\noindent{\textbf{DR-Submodularity:}} 
%A proper subclass of continuous submodular functions, exhibiting diminishing return property, is called DR-submodular 
A continuous submodular function $f$ is {\it DR-submodular} if
\begin{align*}
    f(\boldsymbol{x}+z\boldsymbol{e}_{i})-f(\boldsymbol{x})\le f(\boldsymbol{y}+z\boldsymbol{e}_{i})-f(\boldsymbol{y}),
\end{align*} where $\boldsymbol{e}_{i}$ is the $i$-th basic vector, $\boldsymbol{x}\ge\boldsymbol{y}$ and $z\in \mathbb{R}_{+}$ such that $\boldsymbol{x}+z\boldsymbol{e}_{i}, \boldsymbol{y}+z\boldsymbol{e}_{i}\in\mathcal{X}$. 
When the DR-submodular function $f$ is differentiable, we have $\nabla f(\boldsymbol{x})\le\nabla f(\boldsymbol{y})$ if $\boldsymbol{x}\ge\boldsymbol{y}$ \citep{bian2020continuous}. When $f$ is twice differentiable, the DR-submodularity is also equivalent to
\begin{align*}
    \forall i,j\in[n],\forall\boldsymbol{x}\in\mathcal{X},\frac{\partial^{2}f(\boldsymbol{x})}{\partial x_{i}\partial x_{j}}\le 0.
\end{align*} Furthermore, we call a function $f$ {\it \textbf{weakly} DR-submodular} with parameter $\gamma$, if
\begin{align*}
    \gamma=\inf_{\boldsymbol{x}\le\boldsymbol{y}}\inf_{i\in[n]}\frac{[\nabla f(\boldsymbol{x})]_{i}}{[\nabla f(\boldsymbol{y})]_{i}}.
\end{align*} Note that $\gamma=1$ indicates a differentiable and monotone DR-submodular function.
\subsection{Notations and Concepts}
\noindent{\textbf{Norm:}} $\left\|\cdot\right\|$ is the $\ell_{2}$ norm in Euclidean space.

\noindent{\textbf{Radius and Diameter:}} For any bounded domain $\mathcal{C}\in\mathcal{X}$, the radius $r(\mathcal{C})=\max_{\boldsymbol{x}\in\mathcal{C}} \left\|\boldsymbol{x}\right\|$ and the diameter $\mathrm{diam}(\mathcal{C})=\max_{\boldsymbol{x},\boldsymbol{y}\in\mathcal{C}} \left\|\boldsymbol{x}-\boldsymbol{y}\right\|$.

\noindent{\textbf{Projection:}} We define the projection to the domain $\mathcal{C}$ as $\mathcal{P}_{\mathcal{C}}(\boldsymbol{x})=\arg\min_{\boldsymbol{z}\in\mathcal{C}}\left\|\boldsymbol{x}-\boldsymbol{z}\right\|$.

\noindent{\textbf{Smoothness:}} A differentiable function $f$ is called $L$-$smooth$ if for any $\boldsymbol{x},\boldsymbol{y}\in\mathcal{X}$,
\begin{equation*}
  \left\|\nabla f(\boldsymbol{x})-\nabla f(\boldsymbol{y})\right\|\le L\left\|\boldsymbol{x}-\boldsymbol{y}\right\|.
\end{equation*}
% $\left\|\nabla f(\boldsymbol{x})-\nabla f(\boldsymbol{y})\right\|\le L\left\|\boldsymbol{x}-\boldsymbol{y}\right\|$. %\noindent{\textbf{Mirror Map:}} In section~\ref{sec:gradient_ascent}, we will investigate the mirror ascent algorithm. Thus, we parallel recall some related concepts about the mirror map $\varphi$.If the mapping $\varphi:\mathcal{X}\rightarrow \mathbb{R}$ satisfies: (1)$\varphi$ is $1$-strongly convex and differentiable; (2)For any point $\boldsymbol{y}\in\mathbb{R}^{n}$, there also exists an point $\boldsymbol{x}\in\mathcal{X}$ satisfying $\nabla\varphi(\boldsymbol{x})=\boldsymbol{y}$, we call $\varphi$ a mirror map.
%\noindent{\textbf{Bregman Divergence:}} For the associated mirror map $\varphi$, we define the bregman divergence as 
%$D_{\varphi}(\boldsymbol{x},\boldsymbol{y})=\varphi(\boldsymbol{x})-\varphi(\boldsymbol{y})-\langle\nabla\varphi(\boldsymbol{y}),\boldsymbol{x}-\boldsymbol{y}\rangle$. With the $D_{\varphi}$, we redefine the diameter of $\mathcal{C}$ as $\mathrm{diam}_{\varphi}(\mathcal{C})=\max_{\boldsymbol{x},\boldsymbol{y}\in\mathcal{C}}\sqrt{D_{\varphi}(\boldsymbol{x},\boldsymbol{y})}$ and the mirror projection into $\mathcal{C}$ as $\mathcal{P}_{\mathcal{C}}^{\varphi}(\boldsymbol{x})=\arg\min_{\boldsymbol{z}\in\mathcal{C}}D_{\varphi}(\boldsymbol{z},\boldsymbol{x})$.

\noindent{\textbf{$\alpha$-Regret:}} Finally, we recall the $\alpha$-regret in \citep{chen2018online}. For a $T$-$round$ game, after the algorithm $\mathcal{A}$ choose an action $\boldsymbol{x}_{t}\in\mathcal{X}$ in each round, the adversary reveals the utility function $f_{t}$. The objective of the algorithm $\mathcal{A}$ is to minimize the gap between the accumulative reward and that of the best fixed policy in hindsight with scale parameter $\alpha$, i.e.,
\begin{align*}
    \mathcal{R}_{\alpha}(\mathcal{A},T)=\alpha\max_{\boldsymbol{x}\in\mathcal{X}}\sum_{t=1}^{T}f_{t}(\boldsymbol{x})-\sum_{t=1}^{T}f_{t}(\boldsymbol{x}_{t}).
\end{align*}

\section{Derivation of the Non-oblivious Function}\label{sec:non-oblivious}
%We first recall some celebrated works about the submodular maximization problems. To boost the traditional ($1/2$-approximation) greedy method for the maximum coverage problem over a matroid, \cite{filmus2012power} propose a combinatorial $(1-1/e)$-approximation algorithm via a non-oblivious function allocating the extra weight to the solutions that cover some element more than once. After that, \cite{filmus2014monotone} extends this idea to the general submodular set maximization problem under matroid. Recently, for the continuous submodular maximization with concave regularization, a variant of Frank-Wolfe algorithm based on a special auxiliary function is proposed for boosting the approximation ratio of the submodular part from $1/2$ to $(1-1/e)$ in \citep{mitra2021submodular+}. Inspired by these previous works, in this paper, we devise a uniform and boosting framework by designing an auxiliary function which essentially improves the approximation performance of stationary points. 
In this section, we present in detail how to derive our non-oblivious function, which plays an important role in our boosting framework.
To begin, we recall the definition of stationary points.
\begin{definition}\label{def:1}
A point $\boldsymbol{x}\in\mathcal{C}$ is called a stationary point for function $f:\mathcal{X}\rightarrow \mathbb{R}_{+}$ over the domain $\mathcal{C}\subseteq\mathcal{X}$ if
\begin{equation*}
    \max_{\boldsymbol{y}\in\mathcal{C}}\langle\nabla f(\boldsymbol{x}),\boldsymbol{y}-\boldsymbol{x}\rangle\le 0.
\end{equation*}
% $\max_{\boldsymbol{y}\in\mathcal{C}}\langle\nabla f(\boldsymbol{x}),\boldsymbol{y}-\boldsymbol{x}\rangle\le 0$.
\end{definition}
We make the following assumption throughout this paper.
\begin{assumption}\label{assumption1}\
\begin{enumerate}
   \item[(i)] The $f:\mathcal{X}\rightarrow \mathbb{R}_{+}$ is a monotone, differentiable, weakly DR-submodular function with parameter $\gamma$. So is each $f_{t}$ in the online settings.
   \item[(ii)] We also assume the knowledge of parameter $\gamma$.
    \item[(iii)] Without loss of generality, $f(\boldsymbol{0})=0$. Also, in online settings,  $f_{t}(\boldsymbol{0})=0$ for $t=1,2,\dots,T$.
\end{enumerate}
\end{assumption} 
With this assumption, we have the following result. 
\begin{lemma}[Proof in \cref{proof:lem1}]\label{lemma:2} 
Under Assumption~\ref{assumption1}, for any stationary point $\boldsymbol{x}\in\mathcal{C}$ of $f$, we have 
\begin{equation}
    f(\boldsymbol{x})\ge\frac{\gamma^{2}}{\gamma^{2}+1}\max_{\boldsymbol{y}\in\mathcal{C}}f(\boldsymbol{y}).
\end{equation}
\end{lemma}

\begin{remark}
Lemma~\ref{lemma:2} implies any stationary point of a $\gamma$-weakly DR-submodular function $f$ provides a
$(\frac{\gamma^{2}}{1+\gamma^{2}})$-approximation to the global maximum. 
% As we know, many existing optimization algorithms, such as  projected gradient ascent methods~\citep{hassani2017gradient} and Frank-Wolfe method~\citep{lacoste2016convergence,bian2020continuous}, usually converge to a stationary point. 
As we know, projected gradient ascent method~\citep{hassani2017gradient} with small step size usually converges to a stationary point of $f$, resulting in a $(\frac{\gamma^{2}}{1+\gamma^{2}})$ approximation guarantee.
\end{remark} 

In order to boost these classical algorithms, a natural idea is to design some auxiliary functions whose stationary points achieve better approximation to the global maximum of the problem $\max_{\boldsymbol{x}\in\mathcal{C}}f(\boldsymbol{x})$. That is, we want to find $F:\mathcal{X}\rightarrow \mathbb{R}_{+}$ based on $f$ such that $\langle\boldsymbol{y}-\boldsymbol{x}, \nabla F(\boldsymbol{x})\rangle\ge \beta_{1}f(\boldsymbol{y})-\beta_{2}f(\boldsymbol{x})$,  
where $\beta_{1}/\beta_{2}\ge\frac{\gamma^{2}}{1+\gamma^{2}}$.

Motivated by \citep{feldman2011unified,filmus2012power,filmus2014monotone,harshaw2019submodular,feldman2021guess,mitra2021submodular+}, we consider the function $F(\boldsymbol{x}):\mathcal{X}\rightarrow \mathbb{R}_{+}$ whose gradient at point $\boldsymbol{x}$ allocates different weights to the gradient $\nabla f(z*\boldsymbol{x})$, i.e.,
$\nabla F(\boldsymbol{x})=\int_{0}^{1} w(z)\nabla f(z*\boldsymbol{x})\mathrm{d}z$, assuming that $\nabla f(z*\boldsymbol{x})$ is Lebesgue integrable w.r.t. $z\in[0,1]$, the weight function $w(z)\in C^{1}[0,1]$, and $w(z)\ge 0$. Then, we investigate a property of $\langle \boldsymbol{y}-\boldsymbol{x}, \nabla F(\boldsymbol{x})\rangle$ in the following lemma.

\begin{lemma}[Proof in \cref{proof:lem2}]\label{lemma:3}
For all $\boldsymbol{x},\boldsymbol{y} \in \mathcal{X}$, we have
% \begin{enumerate}
%     \item $\langle\boldsymbol{x},\nabla F(\boldsymbol{x})\rangle\le w(1)f(\boldsymbol{x})-\int_{0}^{1}f(z*\boldsymbol{x})w'(z)\mathrm{d}z$.
%     \item $\langle \boldsymbol{y},\nabla F(\boldsymbol{x})\rangle\ge(\gamma\int_{0}^{1}w(z)\mathrm{d}z)f(\boldsymbol{y})-\int^{1}_{0}\gamma w(z)f(z*\boldsymbol{x})\mathrm{d}z$.
%     \item $\langle \boldsymbol{y}-\boldsymbol{x}, \nabla F(\boldsymbol{x})\rangle \ge(\gamma\int_{0}^{1}w(z)dz)(f(\boldsymbol{y})-\theta(w)f(\boldsymbol{x}))$, where we set $\theta(w)=\max_{f,\boldsymbol{x}}\theta(w,f,\boldsymbol{x})$ and $\theta(w,f,\boldsymbol{x})=\frac{w(1)+\int^{1}_{0} (\gamma w(z)-w'(z))\frac{f(z*\boldsymbol{x})}{f(\boldsymbol{x})}\mathrm{d}z}{\gamma\int_{0}^{1}w(z)\mathrm{d}z}$.
% \end{enumerate}
\begin{align*}
     \langle \boldsymbol{y}-\boldsymbol{x}, \nabla F(\boldsymbol{x})\rangle \ge\left(\gamma\int_{0}^{1}w(z)dz\right)\left(f(\boldsymbol{y})-\theta(w)f(\boldsymbol{x})\right),
\end{align*}
where $\theta(w)=\max_{f,\boldsymbol{x}}\theta(w,f,\boldsymbol{x})$, $\theta(w,f,\boldsymbol{x})=\frac{w(1)+\int^{1}_{0} (\gamma w(z)-w'(z))\frac{f(z*\boldsymbol{x})}{f(\boldsymbol{x})}\mathrm{d}z}{\gamma\int_{0}^{1}w(z)\mathrm{d}z}$ and $f(\boldsymbol{x})>0$. 
\end{lemma}

To improve the approximation ratio, we consider the following factor-revealing optimization problem:
\begin{equation}\label{frop}
   \begin{aligned}
    \min_{w}\theta(w)=\min_{w}&\max_{f,\boldsymbol{x}} \frac{w(1)+\int^{1}_{0} (\gamma w(z)-w'(z))\frac{f(z*\boldsymbol{x})}{f(\boldsymbol{x})}\mathrm{d}z}{\gamma\int_{0}^{1}w(z)\mathrm{d}z}\\
          \rm{s.t.} \ &w(z)\ge 0,\\
          &w(z)\in C^{1}[0,1],\\
          &f(\boldsymbol{x})>0, \\
          &\nabla f(\boldsymbol{x}_1)\ge\gamma\nabla f(\boldsymbol{y}_1)\ge\boldsymbol{0},  \forall \boldsymbol{x}_1\le \boldsymbol{y}_1\in\mathcal{X}.
    \end{aligned} 
\end{equation}
At first glance, problem~\eqref{frop} looks challenging to solve. Fortunately, we could directly find the optimal solution, which is provided in the following theorem.  
\begin{theorem}[Proof in the \cref{proof:thm1}]\label{thm:1} For problem~\eqref{frop}, we have $\hat{w}(z)=e^{\gamma(z-1)}\in\arg\min_{w}\theta(w)$ and $\min_{w}\max_{f,\boldsymbol{x}}\theta(w,f,\boldsymbol{x})=\frac{1}{1-e^{-\gamma}}$. 
\end{theorem}

%\red{TO BE REVISED!} For the better \blue{presentation}, we scale the optimal weight function $w$ in Theorem~\ref{thm:1} with factor $\frac{e^{\gamma}-1}{e^{\gamma}}$, i.e., 
In the following sections, we consider this optimal auxiliary function $F$ with $\nabla F(\boldsymbol{x})=\int_{0}^{1}\hat{w}(z)\nabla f(z*\boldsymbol{x})\mathrm{d}z$, and $\hat{w}(z)=e^{\gamma(z-1)}$. According to the definition of $\theta(w,f,\boldsymbol{x})$ in \cref{lemma:3}, we could derive that $\theta(\hat{w},f,\boldsymbol{x})=\hat{w}(1)/(\gamma\int_{0}^{1}\hat{w}(z)\mathrm{d}z)=1/(1-e^{-\gamma})$ such that $\theta(\hat{w})=1/(1-e^{-\gamma})$. Thus, we have $\langle\boldsymbol{y}-\boldsymbol{x}, \nabla F(\boldsymbol{x})\rangle \ge (1-e^{-\gamma})f(\boldsymbol{y})-f(\boldsymbol{x})$ which implies that any stationary point of $F$ provides a better $(1-e^{-\gamma})$-approximation solution to the problem $\max_{\boldsymbol{x}\in\mathcal{C}}f(\boldsymbol{x})$, in contrast with the stationary points of $f$ itself. 

Next, we investigate some properties of this optimal auxiliary function $F(\boldsymbol{x})$. Following the same terminology in \citep{filmus2012power,filmus2014monotone,mitra2021submodular+}, we also call this $F$ the non-oblivious function.

\subsection{Properties about the Non-Oblivious Function}\label{sec:properties}
Without loss of generality, in this subsection, we assume
$f$ is $L$-smooth with respect to the norm $\left\|\boldsymbol{x}\right\|$, i.e., $\left\|\nabla f(\boldsymbol{x})-\nabla f(\boldsymbol{y})\right\|\le L\left\|\boldsymbol{x}-\boldsymbol{y}\right\|$.
Then, we establish some key properties about the boundness and smoothness of the non-oblivious function $F(\boldsymbol{x})$ in the following theorem.

\begin{theorem}[Proof in \cref{proof:thm2}]\label{thm:2} 
If $f$ is $L$-smooth, and Assumption~\ref{assumption1} holds, we have 
\begin{enumerate}
\item[(i)] $f(\boldsymbol{x})\ge(1-e^{-\gamma})\max_{\boldsymbol{y}\in\mathcal{C}}f(\boldsymbol{y})$, where $\boldsymbol{x}$ is a stationary point for non-oblivious function $F$ over the domain $\mathcal{C}$.
\item[(ii)] $F(\boldsymbol{x})=\int_{0}^{1}\frac{e^{\gamma(z-1)}}{z}f(z*\boldsymbol{x})dz$ and $F(\boldsymbol{x})\le(1+\ln(\tau))(f(\boldsymbol{x})+c)$ for any positive $c\le Lr^{2}(\mathcal{X})$, where $\tau=\max(\frac{1}{\gamma},\frac{Lr^{2}(\mathcal{X})}{c})$. %Therefore, we also conclude that $F$ is a monotone, differentiable, weakly submodular function  with parameter $\gamma$.
\item [(iii)]
% $\left\|\nabla F(\boldsymbol{x})-\nabla F(\boldsymbol{y})\right\|_{*}\le L_{*,\gamma}\left\|\boldsymbol{x}-\boldsymbol{y}\right\|$ where $L_{*,\gamma}=L_{*}\frac{\gamma+e^{-\gamma}-1}{\gamma^2}$.
$F$ is $L_{\gamma}$-smooth where $L_{\gamma}=L\frac{\gamma+e^{-\gamma}-1}{\gamma^2}$.
\end{enumerate}
\end{theorem}

\begin{remark}
\cref{thm:2}.(i) demonstrates that any stationary point of the non-oblivious function $F$ can attain $(1-e^{-\gamma})$-approximation of the global maximum of $f$, which is better than the  $(\frac{\gamma^2}{1+\gamma^2})$-approximation ratio of the stationary points of $f$ itself provided in \cref{lemma:2}.
Moreover, this result sheds light on the possibility of utilizing $F$ to obtain a better approximation than classical gradient ascent method, which motivates our boosting methods in the following section.
\end{remark}

We also investigate how to estimate $\nabla F(\boldsymbol{x})$ with an unbiased stochastic oracle  $\widetilde{\nabla}f(\boldsymbol{x})$, i.e., $\mathbb{E}(\widetilde{\nabla}f(\boldsymbol{x})|\boldsymbol{x})=\nabla f(\boldsymbol{x})$. 
We first introduce a new random variable $\mathbf{Z}$ where Pr$(\mathbf{Z}\le z)=\int_{0}^{z}\frac{\gamma e^{\gamma(u-1)}}{1-e^{-\gamma}}I(u\in[0,1])\mathrm{d}u$ where $I$ is the indicator function. 
When the number $z$ is sampled from r.v. $\mathbf{Z}$, we consider $\frac{1-e^{-\gamma}}{\gamma}\widetilde{\nabla}f(z*\boldsymbol{x})$ as an estimator of $\nabla F(\boldsymbol{x})$ with statistical properties given in the following proposition.
% we could view $\frac{1-e^{-\gamma}}{\gamma}\widetilde{\nabla}f(z*\boldsymbol{x})$ as a good estimation for $\nabla F(\boldsymbol{x})$ with statistical properties given in the following theorem, 

\begin{proposition}[Proof in \cref{proof:prop1}]\label{prop:1} \
\begin{enumerate}
    \item[(i)] If $z$ is sampled from r.v. $\mathbf{Z}$ and $\mathbb{E}(\widetilde{\nabla}f(\boldsymbol{x})|\boldsymbol{x})=\nabla f(\boldsymbol{x})$, we have 
    \begin{equation*}
        \mathbb{E}\left(\left.\frac{1-e^{-\gamma}}{\gamma}\widetilde{\nabla}f(z*\boldsymbol{x})\right|\boldsymbol{x}\right)=\nabla F(\boldsymbol{x}).
    \end{equation*}
    \item[(ii)] If $z$ is sampled from r.v. $\mathbf{Z}$, $\mathbb{E}(\widetilde{\nabla}f(\boldsymbol{x})|\boldsymbol{x})=\nabla f(\boldsymbol{x})$, and $\mathbb{E}(\|\widetilde{\nabla}f(\boldsymbol{x})-\nabla f(\boldsymbol{x})\|^{2}|\boldsymbol{x})\le\sigma^{2}$, we have
       \begin{equation*}
        \mathbb{E}\left(\bigg\|\frac{1-e^{-\gamma}}{\gamma}\widetilde{\nabla}f(z*\boldsymbol{x})-\nabla F(\boldsymbol{x})\bigg\|^{2}\bigg|\boldsymbol{x}\right)\le\sigma^{2}_{\gamma},
        \end{equation*}
     where $\sigma^{2}_{\gamma}=2\frac{(1-e^{-\gamma})^{2}\sigma^{2}}{\gamma^{2}}+\frac{2L^{2}r^{2}(\mathcal{X})(1-e^{-2\gamma})}{3\gamma}$.
\end{enumerate}
\end{proposition}
\begin{remark}
\cref{prop:1} indicates that $\frac{1-e^{-\gamma}}{\gamma}\widetilde{\nabla}f(z*\boldsymbol{x})$ is an unbiased estimator of $\nabla F(\boldsymbol{x})$ with a bounded variance.
\end{remark}

% With all these tools, we \blue{are ready to} present a boosting framework \blue{covering gradient ascent, Frank-Wolfe and online gradient ascent methods}, in which we parallel replace the {\color{orange}(stochastic)} gradient of the original submodular function $f$ with the {\color{orange}(stochastic)} gradient of the non-oblivious function $F$ related to $f$.

% for any existed first-order optimization algorithm converging to the stationary points, in which we parallel replace the (stochastic) gradient of the original submodular function $f$ with the (stochastic) gradient of the non-oblivious function $F$ related to function $f$. 
% Next, we will show this boosting method exactly improving the approximation ratio from $\frac{\gamma^{2}}{1+\gamma^{2}}$ to $1-e^{-\gamma}$ for some known algorithms. Moreover, this boosting methods also make it possible to devise an efficient algorithms for some unexplored submodular optimization settings.

\section{Boosting Framework}\label{sec:boosting_framework}
\begin{algorithm}[t]
	% \small
	\caption{Meta Boosting Protocol}\label{alg:framework}
	% 	\hspace*{0.02in} {\bf Input:} $\forall\widetilde{\nabla} f(\boldsymbol{x})$, $T$, $\eta_{t}$, $c>0$, $\gamma$, $L$, $r(\mathcal{X})$
	\begin{algorithmic}[1]
		\STATE \textbf{Initialize:} any $\boldsymbol{x}_{1}\in\mathcal{X}$.
		\FOR{$t\in [T]$}
		\STATE Sample $z_{t}$ from $\mathbf{Z}$ where $\mathrm{P}(\mathbf{Z}\le z)=\int_{0}^{z}\frac{\gamma e^{\gamma(u-1)}}{1-e^{-\gamma}}I(u\in[0,1])\mathrm{d}u$.
		\STATE Compute $\widetilde{\nabla}F(\boldsymbol{x}_{t}) = \frac{1-e^{-\gamma}}{\gamma}\widetilde{\nabla}f(z_{t}*\boldsymbol{x}_{t})$
		\STATE Update $\boldsymbol{x}_{t+1} = \mathcal{A}(\widetilde{\nabla}F(\boldsymbol{x}_{t}), \boldsymbol{x}_{t})$ \hspace{10em}\COMMENT{$\triangleright$$\mathcal{A}$ to be designed.}
		\ENDFOR
		\STATE 
		% \normalsize
		Option I~(Offline setting): Return $\boldsymbol{x}_{l}$ chosen from $\{\boldsymbol{x}_t\}_{t\in[T]}$ with a probability.
		\STATE 
		% \normalsize
		Option II~(Online setting): Return $\boldsymbol{x}_{t}$ at each round $t\in[T]$.
	\end{algorithmic} 
\end{algorithm}

Up to this point, we present a boosting framework that covers both gradient ascent and online gradient ascent methods. 
We first present a Meta boosting protocol in \cref{alg:framework}, highlighting the key features of the proposed algorithms. We then present several variants of the Meta protocol by employing different basic algorithms $\mathcal{A}$.

As shown in \cref{alg:framework}, the core idea is to leverage the stochastic gradient $\widetilde{\nabla}F(\boldsymbol{x}_{t})$ of the non-oblivious function $F$, instead of the stochastic gradient $\widetilde{\nabla}f(\boldsymbol{x}_{t})$ of the original weakly DR-submodular function $f$. Note that $\widetilde{\nabla}F(\boldsymbol{x}_{t})$ is generated by the sampling method in~\cref{prop:1} (line 3-4 of \cref{alg:framework}).

\subsection{Boosting Gradient Ascent}\label{sec:gradient_ascent}
\begin{algorithm}[t]
	% \small
	\caption{Boosting Gradient Ascent}\label{alg:1}
	\hspace*{0.02in} {\bf Input:} $T$, $\eta_{t}$, $c>0$, $\gamma$, $L$, $r(\mathcal{X})$\\
	\hspace*{0.02in} {\bf Output:} $\boldsymbol{x}_{l}$
	\begin{algorithmic}[1]
		\STATE Set $\triangle_{t}=1$ when $t<T$ and $\triangle_{T}=1+\ln(\tau)$ where $\tau=\max(\frac{1}{\gamma},\frac{r^{2}(\mathcal{X})L}{c})$.
		\STATE Set $\triangle=\sum_{t=1}^{T}\triangle_{t}$
		\STATE \textbf{Initialize} any  $\boldsymbol{x}_{1}\in\mathcal{X}$.
		\FOR{$t\in [T]$}
		% \STATE Sample $z_{t}$ from $\mathbf{Z}$ where $\mathrm{P}(\mathbf{Z}\le z)=\int_{0}^{z}\frac{\gamma e^{\gamma(u-1)}}{1-e^{-\gamma}}I(u\in[0,1])\mathrm{d}u$.
		\STATE Compute $\widetilde{\nabla}F(\boldsymbol{x}_{t})$ according to \cref{alg:framework}
		\STATE Set $\boldsymbol{y}_{t+1}=\boldsymbol{x}_{t}+\eta_{t}\widetilde{\nabla}F(\boldsymbol{x}_{t})$
		\STATE $\boldsymbol{x}_{t+1}=\arg\min_{\boldsymbol{z}\in\mathcal{C}}\left\|\boldsymbol{z}-\boldsymbol{y}_{t+1}\right\|$
		\ENDFOR
		\STATE Choose a number $l\in[T]$ with the distribution $\mathrm{P}(l=t)$=$\frac{\triangle_{t}}{\triangle}$
	\end{algorithmic}
\end{algorithm}

% In the previous section, we devise a non-oblivious function $F$ whose stationary points achieve the $(1-e^{-\gamma})$-approximation to the global maximum. With this powerful tool, we propose a variant of gradient ascent (Algorithm~\ref{alg:1}) for the continuous DR-submodular maximization problems, i.e., $\max_{\boldsymbol{x}\in\mathcal{C}}f(\boldsymbol{x})$
%  where $\mathcal{C}\subseteq\mathcal{X}$ is a bounded convex set. Meanwhile, we assume an unbiased gradient oracle $\widetilde{\nabla}f(\boldsymbol{x})$, i.e., $\mathbb{E}(\widetilde{\nabla}f(\boldsymbol{x})|\boldsymbol{x})=\nabla f(\boldsymbol{x})$.

% \noindent In general, it \blue{could be costly} to evaluate the exact gradient $\nabla F(\boldsymbol{x})=\int_{0}^{1}e^{\gamma(z-1)}\nabla f(z*\boldsymbol{x})\mathrm{d}z$. Hence, in \cref{alg:1}, we first generate an unbiased \blue{estimator of $\nabla F(\boldsymbol{x})$}, according to the sampling method in \cref{thm:3}. 

In this subsection, we propose a boosting gradient ascent method  under the offline scenario for the stochastic submodular maximization problem. In particular, we employ the classical stochastic projected gradient ascent method in the Meta boosting protocol and describe the boosting gradient ascent method in \cref{alg:1}.

As demonstrated in \cref{alg:1}, in each iteration, after calculating $\widetilde{\nabla} F(\boldsymbol{x})$, we make the standard projected gradient step to update $\boldsymbol{x}$. Finally, we return $\boldsymbol{x}_l$ chosen from $\{\boldsymbol{x}_t\}_{t\in[T]}$ with the given distribution.

With the previous outcomes, we establish the convergence result for \cref{alg:1}.

% \begin{theorem}\label{thm:4}Assume $f$ is $L$-$smooth$. Meanwhile, the gradient oracle is unbiased $\mathbb{E}(\widetilde{\nabla}f(\boldsymbol{x})|\boldsymbol{x})=\nabla f(\boldsymbol{x})$ and has a bounded variance $\mathbb{E}(\|\widetilde{\nabla}f(\boldsymbol{x})-\nabla f(\boldsymbol{x})\|_{2}^{2}|\boldsymbol{x})\le\sigma^{2}$ for all points $x\in\mathcal{X}$. Let $\eta_{t}=\frac{1}{\frac{\sigma_{\gamma}\sqrt{t}}{\mathrm{diam}(\mathcal{C})}+L_{\gamma}}$ in the Algorithm~\ref{alg:1}, then we have
%      \begin{align*}
%         &\mathbb{E}(f(\boldsymbol{x}_{l}))\ge\bigg(1-e^{-\gamma}-\frac{1+\ln(\tau)}{T+\ln(\tau)}\bigg)OPT\\&-\frac{(\mathrm{diam}^{2}(\mathcal{C})L_{\gamma}+2\sigma_{\gamma}\mathrm{diam}(\mathcal{C})\sqrt{T})+(1+\ln(\tau))c}{T+\ln(\tau)},
%      \end{align*}
%  where $OPT=\max_{\boldsymbol{x}\in\mathcal{C}}f(\boldsymbol{x})$,  $\sigma^{2}_{\gamma}=2\frac{(1-e^{-\gamma})^{2}\sigma^{2}}{\gamma^{2}}+\frac{2L^{2}r^{2}(\mathcal{X})(1-e^{-2\gamma})}{3\gamma}$, $L_{\gamma}=L\frac{\gamma+e^{-\gamma}-1}{\gamma^2}$,
%  $\tau=\max(\frac{1}{\gamma},\frac{r^{2}(\mathcal{X})L}{c})$ and $c\le r(\mathcal{X})^{2}L$.
% \end{theorem}
% When $c=O(1)$, we can obtain the following corollary.

\begin{theorem}[Proof in \cref{Appendix:B}]\label{thm:4}Assume $\mathcal{C}\in\mathcal{X}$ is a bounded convex set and $f$ is $L$-$smooth$, and the gradient oracle $\widetilde{\nabla}f(\boldsymbol{x})$ is unbiased with $\mathbb{E}(\|\widetilde{\nabla}f(\boldsymbol{x})-\nabla f(\boldsymbol{x})\|^{2}|\boldsymbol{x})\le\sigma^{2}$. Let $\eta_{t}=\frac{1}{\frac{\sigma_{\gamma}\sqrt{t}}{\mathrm{diam}(\mathcal{C})}+L_{\gamma}}$ and $c=O(1)$ in Algorithm~\ref{alg:1}, then we have
 \begin{align*}
    \mathbb{E}(f(\boldsymbol{x}_{l}))\ge\bigg(1-e^{-\gamma}-O\Big(\dfrac{1}{T}\Big)\bigg)OPT-O\Big(\dfrac{1}{\sqrt{T}}\Big),
 \end{align*}
 where $OPT=\max_{\boldsymbol{x}\in\mathcal{C}}f(\boldsymbol{x})$.
\end{theorem}
\begin{remark}
\cref{thm:4} shows that after $O(1/\epsilon^{2})$ iterations, the boosting stochastic gradient ascent achieves $(1-1/e-\epsilon^{2})OPT-\epsilon$, which efficiently improves the $(1/2)$-approximation guarantee of classical stochastic gradient ascent~\citep{hassani2017gradient} for continuous DR-submodular maximization. Moreover, we highlight that the overall gradient complexity is $O(1/\epsilon^2)$ which is optimal~\citep{hassani2020stochastic} under the stochastic setting.
\end{remark}
%In contrast with Frank-Wolfe algorithms, the projected gradient methods are robust to the changes of optimization settings. In most cases, it's easy to obtain a theoretical guarantee solution via dynamically replacing the $\frac{1-e^{-\gamma}}{\gamma}\widetilde{\nabla}f(z_{t}*\boldsymbol{x}_{t})$ in Algorithm~\ref{alg:1}. However, when facing complex settings, we have to seek other techniques to guarantee the convergence of Frank-Wolfe algorithms, such as the variance reduction in the (offline) stochastic setting~\citep{mokhtari2018conditional}, the meta action in the online setting~\citep{chen2018online}, and the block procedure in online bandit setting~\citep{zhang2019online}. Next, we will consider a complicated online setting for continuous DR-submodular maximization, under which a variant of boosting gradient ascent algorithm also has good theoretical performance. 

% \subsection{Boosting Frank-Wolfe}
% \input{sections/Boosting_FW}

\subsection{Online Boosting Delayed Gradient Ascent}\label{sec:delay}
 %In this subsection, 
% {\color{red} to show the power of our boosting framework}, 
%we consider an unexplored online setting, namely, feedback delays. 
%The feedback delays widely exist in many important real-world applications. Considering online advertising \citep{mehta2007adwords}: after displaying an ad, users may not promptly click or even ignore this ad. Before receiving new feedback, an online algorithm has to serve massive incoming ads. Furthermore, in the famous Influence Maximization problem \citep{kempe2003maximizing} of selecting some influential nodes in social network to promote products, due to the time taken in the propagation of information \citep{chen2012time}, we also have to launch a new viral marketing in the social network before the outcome of the past promotion is revealed.

In this section, we consider the online setting with delayed feedback. To begin, recall the process of classical online optimization. In round $t$, after picking an action $\boldsymbol{x}_{t}\in\mathcal{C}$, the environment~(adversary) gives a utility $f_{t}(\boldsymbol{x}_{t})$ and permits the access to the stochastic gradient of $f_{t}$.
%To begin, recall the process of online optimization. At each round, 
%after we pick an action $\boldsymbol{x}_{t}\in\mathcal{C}$, the environment~(adversary) gives us a utility $f_{t}(\boldsymbol{x}_{t})$ and permits the access to the stochastic gradient of $f_{t}$. 
The objective is to minimize the $\alpha$-regret for $T$ planned rounds.
Then, we turn to the (adversarial) feedback delays phenomenon \citep{quanrud2015online} in our online stochastic submodular maximization problem. That is, instead of the prompt feedback, the information about the stochastic gradient of $f_{t}$ could be delivered at the end of round $(t+d_{t}-1)$, where $d_{t}\in\mathbb{Z}_{+}$ is a positive integer delay for round $t$. For instance, the standard online setting sets all $d_{t}=1$~\citep{hazan2019introduction}. 

Next, we introduce some useful notations. We denote the feedback given at the end of round $t$ as $\mathcal{F}_{t}=\{u\in[T]: u+d_{u}-1=t\}$ and $D=\sum_{t=1}^{T}d_{t}$. Hence, at the end of round $t$, we only have access to the stochastic gradients of past $f_{s}$ where $s\in\mathcal{F}_{t}$. 
% Recently, \citet{quanrud2015online} proposed an online delayed gradient descent algorithm for convex objectives, which achieves a regret of $O(\sqrt{D})$. However, for online setting without delays, the \citet{chen2018online} have proved the online gradient ascent with a limited (1/2)-regret of $O(\sqrt{T})$.

To improve the state-of-the-art $1/2$ approximation ratio of online gradient ascent and tackle the adversarial delays simultaneously, we employ the online delayed gradient algorithm~\citep{quanrud2015online} in the Meta boosting protocol, in which we utilize the stochastic gradient of the non-oblivious function $F$. 
As shown in \cref{alg:2}, at each round $t$, after querying the stochastic gradient $\widetilde{\nabla}F_{t}(\boldsymbol{x}_{t})$, we apply the received stochastic gradients feedback $\widetilde{\nabla}F_{s}(\boldsymbol{x}_{s})$ $( s\in\mathcal{F}_{t})$ in a standard projection gradient step to update $\boldsymbol{x}_{t}$.

\begin{algorithm}[t]
	% \small
	\caption{Online Boosting Delayed Gradient Ascent}\label{alg:2}
	\hspace*{0.02in} {\bf Input:} $T$, $\eta$, $\gamma$\\
	\hspace*{0.02in} {\bf Output:} $\boldsymbol{x}_{1},\dots,\boldsymbol{x}_{T}$
	\begin{algorithmic}[1]
		\STATE \textbf{Initialize:} any  $\boldsymbol{x}_{1}\in\mathcal{C}$.
		\FOR{$t\in [T]$}
		\STATE Play $\boldsymbol{x}_{t}$, then observe reward $f_{t}(\boldsymbol{x}_{t})$
		% \STATE Sample $z_{t}$ from $\mathbf{Z}$ where $\mathrm{Pr}(\mathbf{Z}\le z)=\int_{0}^{z}\frac{\gamma e^{\gamma(u-1)}}{1-e^{-\gamma}}I(u\in[0,1])\mathrm{d}u$.
		\STATE Sample $z_{t}$ according to \cref{alg:framework} and Query $\widetilde{\nabla}F_t(\boldsymbol{x}_{t})=\frac{1-e^{-\gamma}}{\gamma}\widetilde{\nabla}f_{t}(z_{t}*\boldsymbol{x}_{t})$
		\STATE Receive feedback $\widetilde{\nabla}F_s(\x_s)$, where $s\in \mathcal{F}_t$
		\STATE $\boldsymbol{y}_{t+1}=\boldsymbol{x}_{t}+\eta\sum_{s\in\mathcal{F}_{t}}\widetilde{\nabla}F_s(\boldsymbol{x}_{s})$
		\STATE $\boldsymbol{x}_{t+1}=\arg\min_{\boldsymbol{z}\in\mathcal{C}}\left\|\boldsymbol{z}-\boldsymbol{y}_{t+1}\right\|$
		\ENDFOR
	\end{algorithmic}
\end{algorithm}

We provide the regret bound of Algorithm~\ref{alg:2}.% in the following theorem.
\begin{theorem}[Proof in \cref{Appendix:D}]\label{thm:5}
Assume $\mathcal{C}\subseteq\mathcal{X}$ is a bounded convex set and each $f_{t}$ is monotone, differentiable, and weakly DR-submodular with $\gamma$. Meanwhile, the gradient oracle is unbiased $\mathbb{E}(\widetilde{\nabla}f_{t}(\boldsymbol{x})|\boldsymbol{x})=\nabla f_{t}(\boldsymbol{x})$ and $\max_{t\in[T]}(\|\widetilde{\nabla}F_{t}(\boldsymbol{x}_{t})\|)=\frac{1-e^{-\gamma}}{\gamma}\max_{t\in[T]}(\|\widetilde{\nabla}f_{t}(\boldsymbol{x}_{t})\|)$ . Let $\eta=\frac{diam(\mathcal{C})}{\max_{t\in[T]}(\left\|\widetilde{\nabla}F_{t}(\boldsymbol{x}_{t})\right\|)\sqrt{D}}$ in Algorithm~\ref{alg:2}, then we have
\begin{align*}
  (1-e^{-\gamma})\max_{\boldsymbol{x}\in\mathcal{C}}\sum_{t=1}^{T}f_{t}(\boldsymbol{x})-\mathbb{E}(\sum_{t=1}^{T}f_{t}(\boldsymbol{x}_{t}))=O(\sqrt{D}),  
\end{align*}
where $D=\sum_{i=1}^{T}d_{t}$ and $d_{t}\in\mathbb{Z}_{+}$ is a positive delay for the information about $f_{t}$.
\end{theorem}

\begin{remark}
When no delay exists, i.e., $d_{t}=1$ for all $t$, \cref{thm:5} says that the online boosting gradient ascent achieves a ($1-e^{-\gamma}$)-regret of $O(\sqrt{T})$. 
% In contrast with previous Frank-Wolfe algorithms~\citep{chen2018online,chen2018projection,zhang2019online}, 
To the best of our knowledge, this is the first result achieving a $(1-e^{-\gamma})$-regret of $O(\sqrt{T})$ with $O(1)$ stochastic gradient queries for each submodular function $f_{t}$. 
\end{remark}
% Furthermore, under the delays of stochastic gradients, it is hard to analyse these algorithms \citep{chen2018projection,chen2018online,zhang2019online} for online submodular maximization. 
% In the Meta-Frank-Wolfe algorithm \citep{chen2018projection}, at the end of round $t$, the variance reduction technique need to inquire $O(T^{3/2})$ unbiased stochastic gradients of $f_{t}$ to approximate the exact gradient $\nabla f_{t}(\boldsymbol{x}_{t})$. 
% However, when considering the delays of stochastic gradients, we only have access to the stochastic gradients of $f_{s}$~($s\in\mathcal{F}_{t}$) at the end of round $t$. 
% It is challenging to obtain theoretical guarantees via using the stochastic gradients of $f_{s}~(\forall s\in\mathcal{F}_{t}$) to approximate the exact gradients of $f_{t}$, without more information between the (stochastic) gradients of $f_{s}$~($s\in\mathcal{F}_{t}$) and those of $f_{t}$. 
% Hence, it remains an open problem to analyze the theoretical performance of the Meta-Frank-Wolfe in this new online stochastic setting with unknown delays. Instead, the online boosting gradient ascent achieves the optimal $(1-1/e)$-regret of $O(\sqrt{D})$ in this new online setting.
\begin{remark}
Under the delays of stochastic gradients, \cref{thm:5} gives the first regret analysis for the online stochastic submodular maximization problem. 
It is worth mentioning that the $(1-e^{-\gamma})$-regret of $O(\sqrt{D})$ result not only achieves the optimal $(1-e^{-\gamma})$ approximation ratio, but also matches the $O(\sqrt{D})$ regret of online convex optimization with adversarial delays~\citep{quanrud2015online}.
\end{remark}

%\section{Discussions}\label{sec:concave}
%\input{sections/discussion}
\section{Numerical Experiments}\label{sec:experiments}
In this section, we empirically evaluate our proposed boosting algorithms in both offline and online settings by adopting continuous DR-submodular objective functions ($\gamma=1$).

%\orange{The objective functions are continuous DR-submodular, which means $\gamma=1$.}

\subsection{Offline Settings}
We first consider offline DR-submodular maximization problems and compare the following algorithms:\\
\noindent\textbf{Boosting Gradient Ascent~(BGA($B$))}: In the frame   work of Algorithm~\ref{alg:1}, we use the average of $B$ independent stochastic gradients to estimate $\nabla F(x)$ in every iteration.\\
\noindent\textbf{Gradient Ascent~(GA)}: We consider Algorithm 1 in \citet{hassani2017gradient} with step size $\eta_{t}=1/\sqrt{t}$.\\
\noindent\textbf{Continuous Greedy~(CG)}: Algorithm 1 in \cite{bian2017guaranteed}.\\
\noindent\textbf{Stochastic Continuous Greedy~(SCG)}: Algorithm 1 in \citet{mokhtari2018conditional} with $\rho_{t}=1/(t+3)^{2/3}$.\\
%\item Boosting Frank-Wolfe~(boosting FW): In the Algorithm~\ref{alg:3} , we set the $\rho_{t}=1/(t+3)^{2/3}$ and $\eta=1/T^{1/3}$.
\noindent\textbf{Stochastic Continuous Greedy++~(SCG++)}: We consider Algorithm 4.1 in \citet{hassani2020stochastic} where we set the minibatch size $|\mathcal{M}_{0}|=T^{2}$ and  $|\mathcal{M}|=T$ for $T$-round iterations.

\begin{figure*}[t]
\vspace{-1.0em}
\centering
\subfigure[Special Case \label{graph1}]{\includegraphics[width=0.32\linewidth]{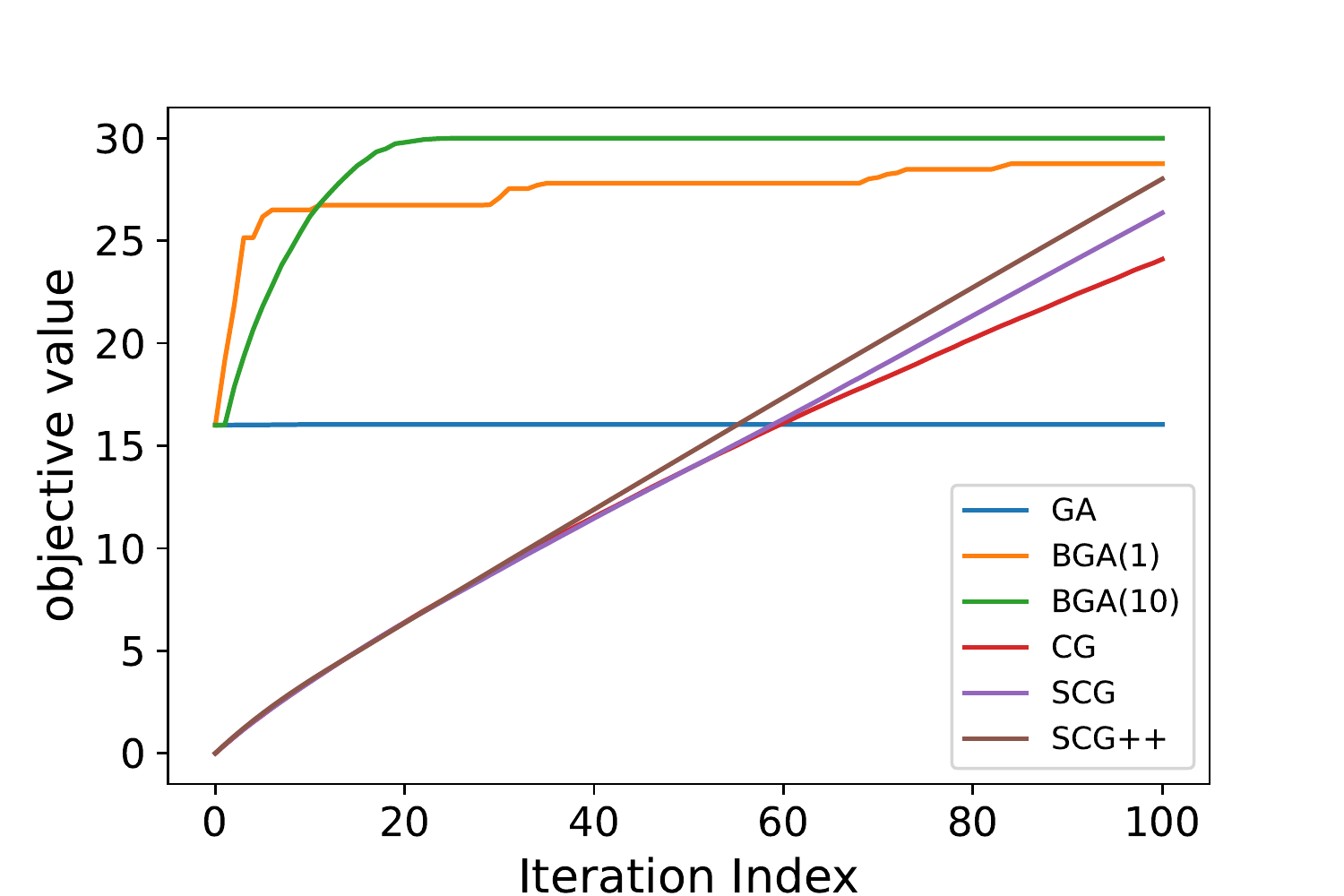}}
\subfigure[Special Case~(origin)\label{graph2}]{\includegraphics[width=0.32\linewidth]{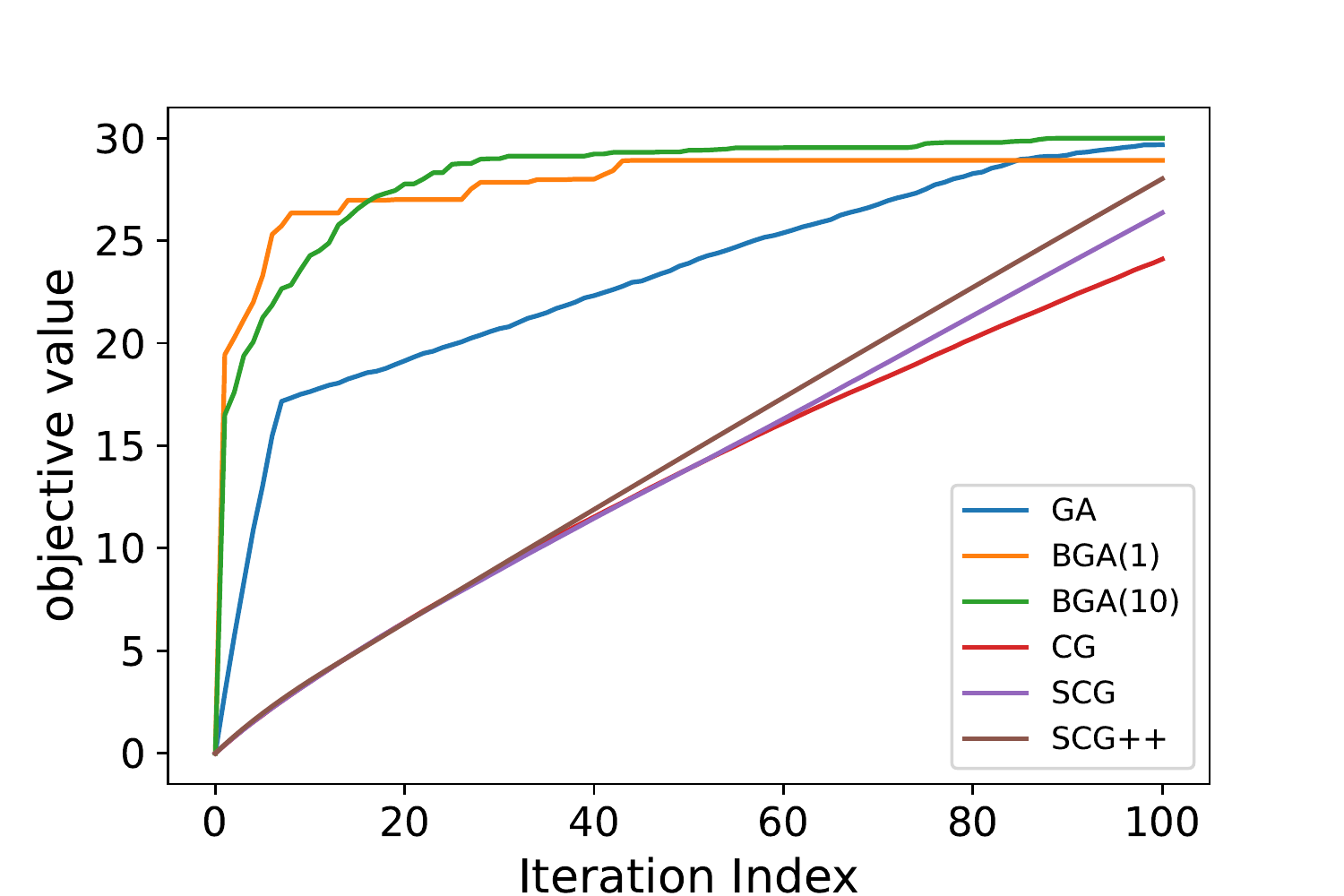}}
\subfigure[Offline QP\label{graph3}]{\includegraphics[width=0.32\linewidth]{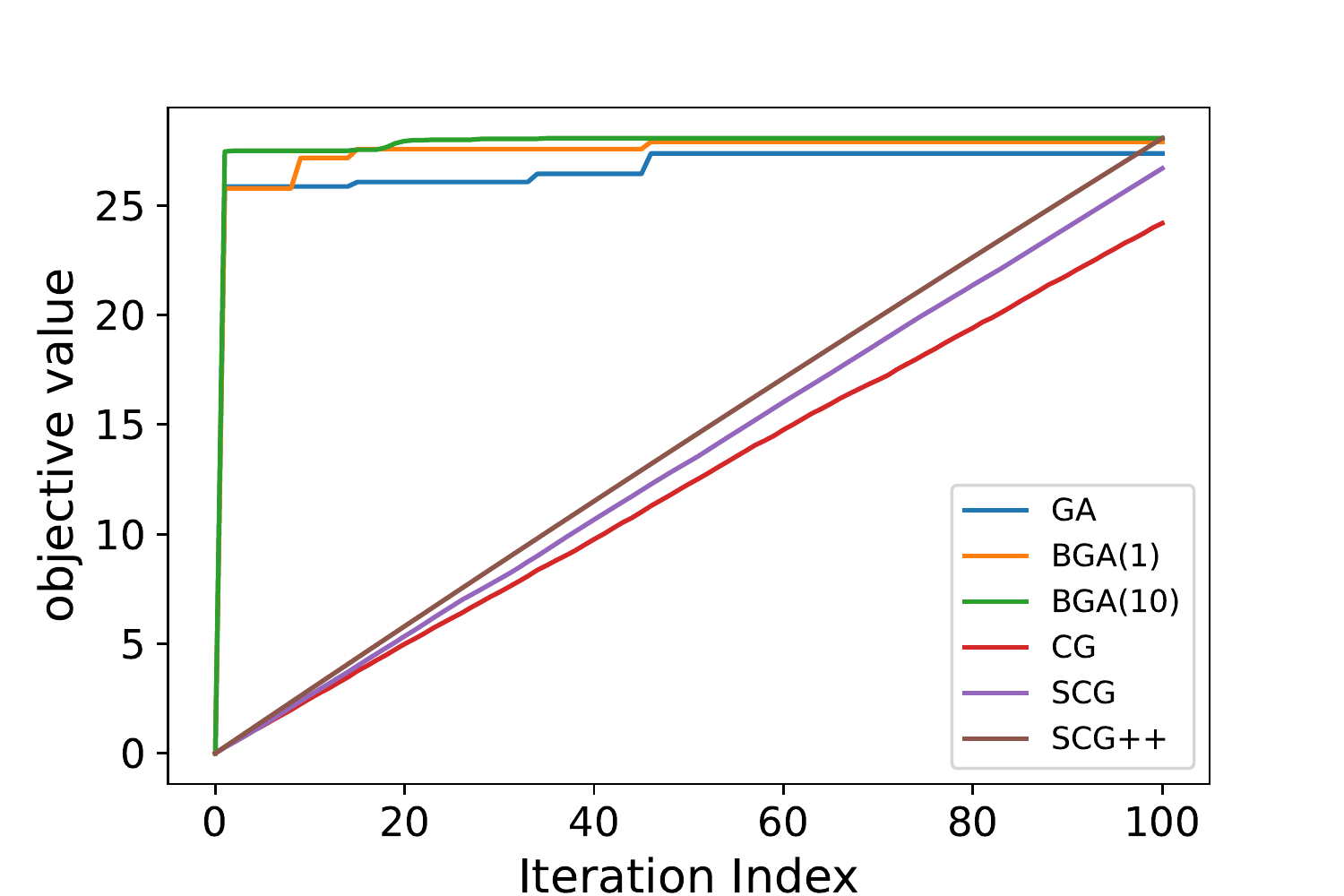}}
\subfigure[Online QP with feedback delays \label{graph4}]{\includegraphics[width=0.35\linewidth]{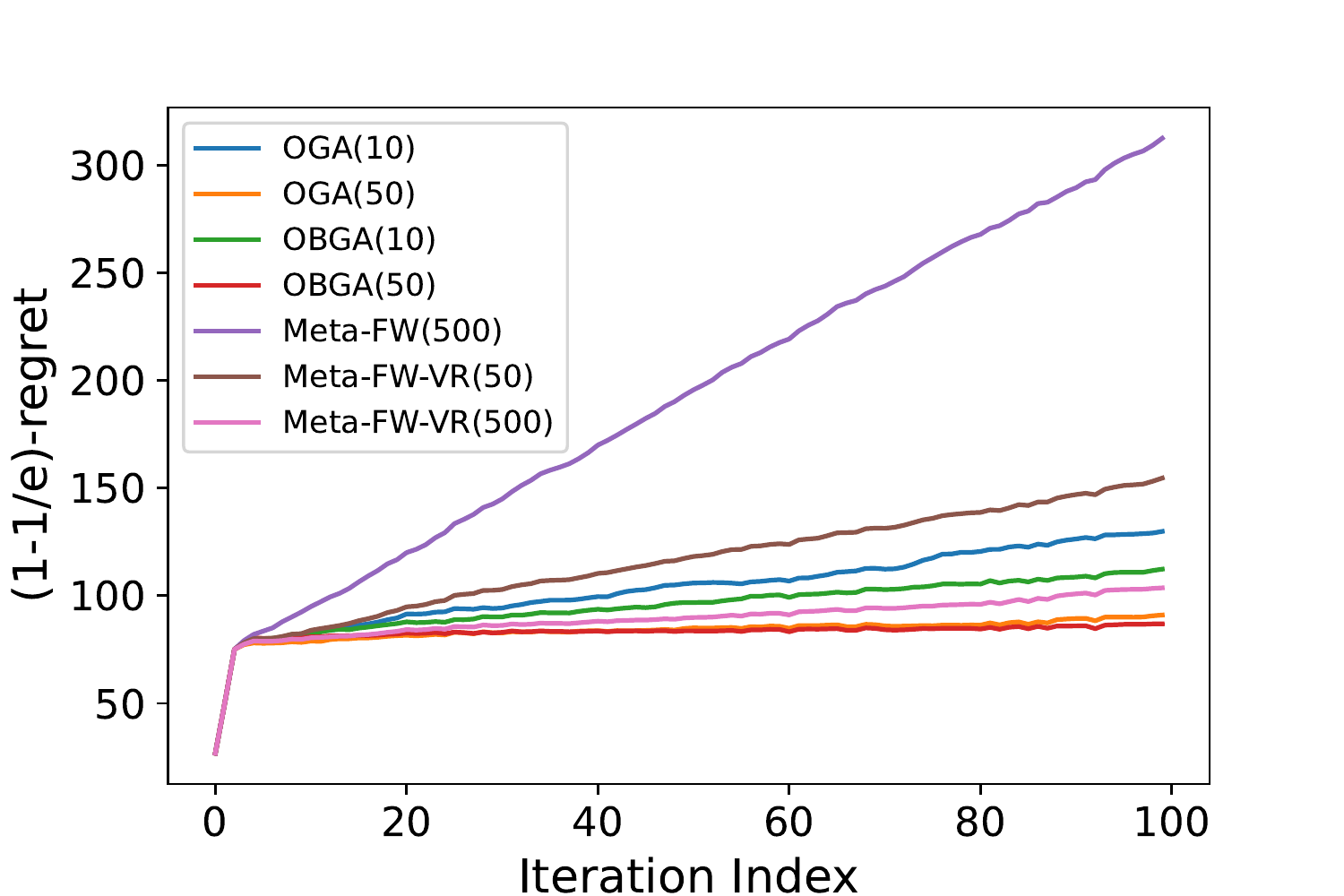}}
\subfigure[Online QP without feedback delays\label{graph5}]{\includegraphics[width=0.35\linewidth]{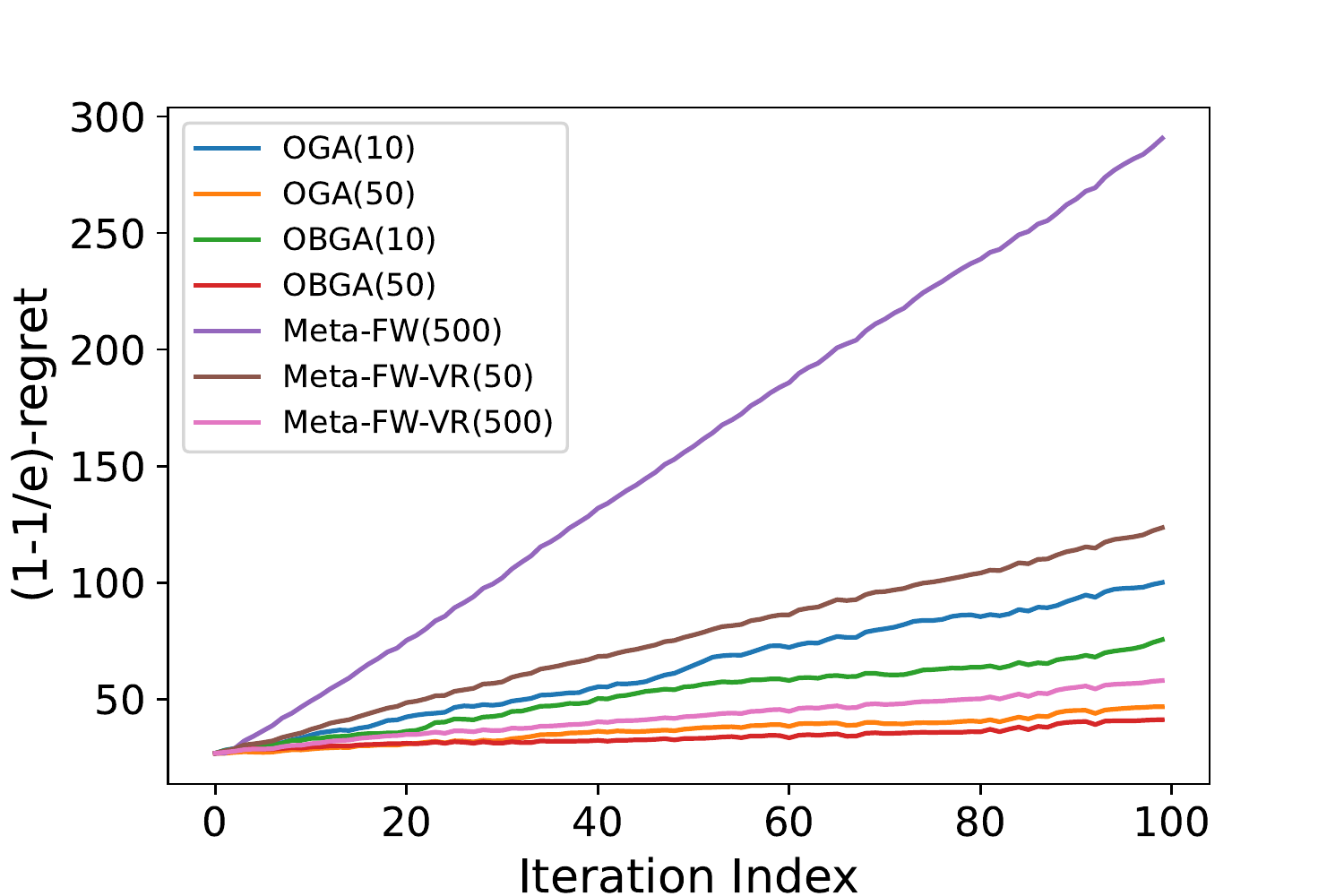}}
\caption{In ~\cref{graph1}, we test the performance of the six algorithms for the special submodular function in \cite{hassani2017gradient} where the GA, BGA and BGA(10) start from $\boldsymbol{x}_{loc}$. Simultaneously, we present the results for all algorithm starting from the origin in ~\cref{graph2}. ~\cref{graph3} show the performance of the algorithms versus the number of iterations in a simulated Non-convex/Non-concave submodular QP. Finally, ~\cref{graph4} and ~\cref{graph5} show the $(1-1/e)$-regret of the seven algorithms, including OGA(10), OGA(50), OBGA(10), OBGA(50), Meta-FW(500), Meta-FW-VR(50), and Meta-FW-VR(500), for the simulated online submodular QP in both delayed setting and standard online setting.}
\vspace{-1.0em}
\end{figure*}

\subsubsection{Special Case}
\citet{hassani2017gradient} introduced a special continuous DR-submodular function $f_{k}$ coming from the multilinear extension of a set cover function. Here,  
$f_{k}(\boldsymbol{x})=k+1-(1-x_{2k+1})\prod_{i=1}^{k}(1-x_{i})-(1-x_{2k+1})(k-\sum_{i=1}^{k}x_{i})+\sum_{i=k+1}^{2k}x_{i}$, where $\boldsymbol{x}=(x_1,x_2,\dots,x_{2k+1})$.
Under the domain $\mathcal{C}=\{\boldsymbol{x}\in[0,1]^{2k+1}: \sum_{i=1}^{2k+1}x_{i}=k\}$, \citet{hassani2017gradient} also
verified that $\boldsymbol{x}_{loc}=(\overbrace{1,1,\dots,1}^{k},0,\dots,0)$ is a local maximum with $(1/2+1/(2k))$-approximation to the global maximum. Thus, if start at $\boldsymbol{x}_{loc}$, theoretically Gradient Ascent \citep{hassani2017gradient} will get stuck at this local maximum point. In our experiment, we set $k=15$ and consider a standard Gaussian noise, i.e., $[\widetilde{\nabla}f(x)]_{i}=[\nabla f(x)]_{i}+\mathcal{N}(0,1)$ for any $i\in[2k]$.  

First, we set the initial point of GA, BGA(1) and BGA(10) to be $\boldsymbol{x}_{loc}$. From Figure~\ref{graph1}, we observe that GA stays at $\boldsymbol{x}_{loc}$ as expected. Instead, BGA(1) and BGA(10) escape the local maximum $\boldsymbol{x}_{loc}$ and achieve near-optimal objective values. Then, we run all algorithms from the origin and present the results in Figure~\ref{graph2}. It shows that GA, starting from the origin, performs much better than from a local maximum. Compared with GA, BGA(1) and BGA(10) converge to the optimal point $\boldsymbol{x}^{*}=(0,\dots,0,\overbrace{1,1,\dots,1}^{k+1})$ more rapidly. Both \cref{graph1} and \cref{graph2} show that BGA(1) and BGA(10) also perform better than Frank-Wolfe-type algorithms with respect to the convergence rate and the objective value. %Therefore, the local search policy based on the non-oblivious function efficiently solves the continuous DR-submodular maximization problems adopted in experiments. 

\subsubsection{Non-Convex/Non-Concave Quadratic Programming}\label{sec:qp}
We consider the quadratic objective $f(\boldsymbol{x}) = \frac{1}{2}\boldsymbol{x}^{T}\boldsymbol{H}\boldsymbol{x} + \boldsymbol{h}^{T}\boldsymbol{x}$ and constraints $P=\{\boldsymbol{x}\in \mathbb{R}^{n}_{+} | \boldsymbol{A}\boldsymbol{x}\le\boldsymbol{b}, \boldsymbol{0}\le\boldsymbol{x}\le\boldsymbol{u}, \boldsymbol{A}\in \mathbb{R}^{m\times n}_{+}, \boldsymbol{b}\in \mathbb{R}_{+}^{m}\}$. Following \cite{bian2017guaranteed}, we choose the matrix $\boldsymbol{H}\in \mathbb{R}^{n\times n}$ to be a randomly generated symmetric matrix with entries uniformly distributed in $[-1,0]$, and the matrix $\boldsymbol{A}$ to be a random matrix with entries uniformly distributed in $[0,1]$. 
%Due to the definition in section~\ref{sec:pre}, 
It can be verified that $f$ is a continuous DR-submodular function. We also set $\boldsymbol{b}=\boldsymbol{u}=\boldsymbol{1}$, $m=12$, and $n=25$. To ensure the monotonicity, we set $\boldsymbol{h}=-\boldsymbol{H}^{T}\boldsymbol{u}$. Thus, the objective becomes $f(x)=(\frac{1}{2}\boldsymbol{x}-\boldsymbol{u})^{T}\boldsymbol{H}\boldsymbol{x}$. Similarly, we also consider the Gaussian noise for gradient, i.e., $[\widetilde{\nabla}f(\boldsymbol{x})]_{i}=[\nabla f(\boldsymbol{x})]_{i}+\delta\mathcal{N}(0,1)$ for any $i\in[n]$. We consider $\delta=5$ and start all algorithms from the origin.

% In our experiment, we find it hard to simulate a submodular QP with $(1/2+\epsilon)$ local maximum, as the special case in \cite{hassani2017gradient}. 
As shown in Figure~\ref{graph3}, BGA(1) and BGA(10) converge faster than GA and achieve nearly the same objective values as GA after $100$ iterations. Similar to the previous experiment, BGA(1) and BGA(10) exceed Frank-Wolfe-type algorithms with respect to the convergence rate.
% Moreover, we can observe that the boosting policy via the non-oblivious function brings significant benefits to the Frank-Wolfe algorithms. %for instance, the boosting Frank-Wolfe is much better than the previous stochastic continuous greedy methods with respect to the speed and achievement.

\subsection{Online Settings}
We also consider Online DR-submodular Maximization with/without adversarial delays. Here, we present a list of algorithms to be compared in these settings:\\
\noindent\textbf{Meta-Frank-Wolfe~(Meta-FW($K$))}: We consider Algorithm 1 in \cite{chen2018online} and initialize $K$ online gradient descent oracles~\citep{zinkevich2003online,hazan2019introduction} with step size $1/\sqrt{T}$.\\
\noindent\textbf{Stochastic Meta-Frank-Wolfe~(Meta-FW-VR($K$))}: We consider Algorithm 1 in \citep{chen2018projection} with the $\rho_{t}=1/(t+3)^{2/3}$ and $K$ online gradient descent oracles with step size $1/\sqrt{T}$.\\
\noindent\textbf{Online Gradient Ascent~(OGA($B$))}: The delayed gradient ascent algorithm in \citep{quanrud2015online} with step size $1/\sqrt{T}$. We use $B$ independent samples to estimate $\nabla f_{t}(\boldsymbol{x}_{t})$ at each round.\\
\noindent\textbf{Online Boosting Gradient Ascent~(OBGA($B$))}:  We consider Algorithm~\ref{alg:2} with the step size $\eta_{t}=1/\sqrt{T}$ and use the average of $B$ independent samples to estimate the gradient at each round.

% \subsubsection{Online Non-Convex/Non-Concave Quadratic Programming with Unknown Delays}
The same as \cref{sec:qp}, we first generate $T=100$ quadratic objective functions $f_{1}, f_2, \dots, f_{T}$. The symmetric random matrix $H_{t}$, corresponding to $f_t$, is uniformly generated from $[-1,0]^{n\times n}$ for $t=1,\ldots, T$, and the matrix $\boldsymbol{A}$ in constraint is randomly generated from the uniform distribution in $[0,1]^{m\times n}$. We also add the Gaussian noise for the gradient of each $f_{t}$, i.e., $[\widetilde{\nabla}f_t(\boldsymbol{x})]_{i}=[\nabla f_t(\boldsymbol{x})]_{i}+\delta\mathcal{N}(0,1)$ with $\delta=5$ for any $i\in[n]$. To simulate the feedback delays, we generate a uniform random number $d_{t}$ from $\{1,2,3,4,5\}$ for the stochastic gradient information of $f_{t}$.  

% As demonstrated in \cref{graph4} and \cref{graph5}, our proposed OBGA(50) has the minimum $(1-1/e)$-regret in both delayed setting and the standard online setting~\citep{hazan2019introduction}, compared to the OGA, Meta-FW and Meta-FW-VR algorithms. 
% We also find that increasing the minibatch size, such as $K$ and $B$, will reduce the $(1-1/e)$-regret. Specially, the delays of feedback result in a larger $(1-1/e)$-regret for all algorithms. 
We present the $(1-1/e)$-regret of algorithms for the delayed setting and the standard online setting~\citep{hazan2019introduction} in \cref{graph4} and \cref{graph5}, respectively. Under both scenarios, our proposed OBGA with sample size $B=50$ exhibits the lowest regret among all algorithms. With the same sample size $B=10$ and $50$, OBGA consistently achieves lower regrets than OGA, which confirms the effectiveness of our boosting framework.

%Furthermore, the boosting gradient ascent theoretically achieves the (optimal) $(1-1/e)$-regret of $\sqrt{D}$, compared with the $(1/2)$-approximation online gradient ascent algorithm. 
%Moreover, the results show that the delay of stochastic gradients have more impact on the Meta-Frank-Wolfe algorithm and its variants to the proposed online algorithm.  These results support the presented regret of \cref{alg:2}.

\section{Conclusion}
In this paper, based on a novel non-oblivious function, we present a boosting framework, covering boosting gradient ascent and online boosting delayed gradient ascent, for the stochastic continuous submodular maximization problem, under both offline and online settings. 
In the offline scenario, our boosting gradient ascent provides $(1-e^{-\gamma}-\epsilon^{2})$-approximation guarantees after $O(1/\epsilon^{2})$ iterations.
Under the online setting, we are the first to consider delayed feedback for online submodular maximization problems.
Moreover, when no delay exists, our online boosting delayed gradient ascent is the first result to guarantee $(1-e^{-\gamma})$-approximation with $O(\sqrt{T})$ regret, where at each round we only estimate stochastic gradient $O(1)$ times. Numerical experiments demonstrate the superior performance of our algorithms.

\section*{Acknowledgements}
The authors would like to thank the anonymous reviewers for their helpful comments. Zhang and Yang's research is supported in part by the Hong Kong Research Grants Council (ECS 21214720), City University of Hong Kong (Project 9610465) and Alibaba Group through Alibaba Innovative Research (AIR) Program.

\bibliography{references}

\begin{thebibliography}{42}
\providecommand{\natexlab}[1]{#1}
\providecommand{\url}[1]{\texttt{#1}}
\expandafter\ifx\csname urlstyle\endcsname\relax
  \providecommand{\doi}[1]{doi: #1}\else
  \providecommand{\doi}{doi: \begingroup \urlstyle{rm}\Url}\fi

\bibitem[Alimonti(1994)]{alimonti1994new}
Paola Alimonti.
\newblock New local search approximation techniques for maximum generalized
  satisfiability problems.
\newblock In \emph{Italian Conference on Algorithms and Complexity}, pages
  40--53. Springer, 1994.

\bibitem[Arora et~al.(2016)Arora, Ge, Kannan, and Moitra]{arora2016computing}
Sanjeev Arora, Rong Ge, Ravi Kannan, and Ankur Moitra.
\newblock Computing a nonnegative matrix factorization---provably.
\newblock \emph{SIAM Journal on Computing}, 45\penalty0 (4):\penalty0
  1582--1611, 2016.

\bibitem[Bertsekas(2015)]{bertsekas2015convex}
Dimitri Bertsekas.
\newblock \emph{Convex optimization algorithms}.
\newblock Athena Scientific, 2015.

\bibitem[Bian et~al.(2017)Bian, Mirzasoleiman, Buhmann, and
  Krause]{bian2017guaranteed}
Andrew~An Bian, Baharan Mirzasoleiman, Joachim Buhmann, and Andreas Krause.
\newblock Guaranteed non-convex optimization: Submodular maximization over
  continuous domains.
\newblock In \emph{Artificial Intelligence and Statistics}, pages 111--120.
  PMLR, 2017.

\bibitem[Bian et~al.(2020)Bian, Buhmann, and Krause]{bian2020continuous}
Yatao Bian, Joachim~M Buhmann, and Andreas Krause.
\newblock Continuous submodular function maximization.
\newblock \emph{arXiv preprint arXiv:2006.13474}, 2020.

\bibitem[Chekuri et~al.(2014)Chekuri, Vondr{\'a}k, and
  Zenklusen]{chekuri2014submodular}
Chandra Chekuri, Jan Vondr{\'a}k, and Rico Zenklusen.
\newblock Submodular function maximization via the multilinear relaxation and
  contention resolution schemes.
\newblock \emph{SIAM Journal on Computing}, 43\penalty0 (6):\penalty0
  1831--1879, 2014.

\bibitem[Chen et~al.(2018{\natexlab{a}})Chen, Harshaw, Hassani, and
  Karbasi]{chen2018projection}
Lin Chen, Christopher Harshaw, Hamed Hassani, and Amin Karbasi.
\newblock Projection-free online optimization with stochastic gradient: From
  convexity to submodularity.
\newblock In \emph{International Conference on Machine Learning}, pages
  814--823. PMLR, 2018{\natexlab{a}}.

\bibitem[Chen et~al.(2018{\natexlab{b}})Chen, Hassani, and
  Karbasi]{chen2018online}
Lin Chen, Hamed Hassani, and Amin Karbasi.
\newblock Online continuous submodular maximization.
\newblock In \emph{International Conference on Artificial Intelligence and
  Statistics}, pages 1896--1905. PMLR, 2018{\natexlab{b}}.

\bibitem[Chen et~al.(2012)Chen, Lu, and Zhang]{chen2012time}
Wei Chen, Wei Lu, and Ning Zhang.
\newblock Time-critical influence maximization in social networks with
  time-delayed diffusion process.
\newblock In \emph{Twenty-Sixth AAAI Conference on Artificial Intelligence},
  2012.

\bibitem[Das and Kempe(2011)]{das2011submodular}
Abhimanyu Das and David Kempe.
\newblock Submodular meets spectral: greedy algorithms for subset selection,
  sparse approximation and dictionary selection.
\newblock In \emph{International Conference on Machine Learning}, pages
  1057--1064, 2011.

\bibitem[Du et~al.(2019)Du, Zhai, Poczos, and Singh]{du2018gradient}
Simon~S. Du, Xiyu Zhai, Barnabas Poczos, and Aarti Singh.
\newblock Gradient descent provably optimizes over-parameterized neural
  networks.
\newblock In \emph{International Conference on Learning Representations}, 2019.

\bibitem[Elenberg et~al.(2018)Elenberg, Khanna, Dimakis, and
  Negahban]{elenberg2018restricted}
Ethan~R Elenberg, Rajiv Khanna, Alexandros~G Dimakis, and Sahand Negahban.
\newblock Restricted strong convexity implies weak submodularity.
\newblock \emph{The Annals of Statistics}, 46\penalty0 (6B):\penalty0
  3539--3568, 2018.

\bibitem[Feldman(2021)]{feldman2021guess}
Moran Feldman.
\newblock Guess free maximization of submodular and linear sums.
\newblock \emph{Algorithmica}, 83\penalty0 (3):\penalty0 853--878, 2021.

\bibitem[Feldman et~al.(2011)Feldman, Naor, and Schwartz]{feldman2011unified}
Moran Feldman, Joseph Naor, and Roy Schwartz.
\newblock A unified continuous greedy algorithm for submodular maximization.
\newblock In \emph{2011 IEEE 52nd Annual Symposium on Foundations of Computer
  Science}, pages 570--579. IEEE, 2011.

\bibitem[Filmus and Ward(2012)]{filmus2012power}
Yuval Filmus and Justin Ward.
\newblock The power of local search: Maximum coverage over a matroid.
\newblock In \emph{29th Symposium on Theoretical Aspects of Computer Science},
  volume~14, pages 601--612. LIPIcs, 2012.

\bibitem[Filmus and Ward(2014)]{filmus2014monotone}
Yuval Filmus and Justin Ward.
\newblock Monotone submodular maximization over a matroid via non-oblivious
  local search.
\newblock \emph{SIAM Journal on Computing}, 43\penalty0 (2):\penalty0 514--542,
  2014.

\bibitem[Fisher et~al.(1978)Fisher, Nemhauser, and Wolsey]{fisher1978analysis}
Marshall~L Fisher, George~L Nemhauser, and Laurence~A Wolsey.
\newblock An analysis of approximations for maximizing submodular set
  functions—ii.
\newblock In \emph{Polyhedral Combinatorics}, pages 73--87. Springer, 1978.

\bibitem[Fujishige(2005)]{fujishige2005submodular}
Satoru Fujishige.
\newblock \emph{Submodular functions and optimization}.
\newblock Elsevier, 2005.

\bibitem[Ge et~al.(2016)Ge, Lee, and Ma]{ge2016matrix}
Rong Ge, Jason~D Lee, and Tengyu Ma.
\newblock Matrix completion has no spurious local minimum.
\newblock In \emph{Advances in Neural Information Processing Systems}, pages
  2973--2981, 2016.

\bibitem[Harshaw et~al.(2019)Harshaw, Feldman, Ward, and
  Karbasi]{harshaw2019submodular}
Chris Harshaw, Moran Feldman, Justin Ward, and Amin Karbasi.
\newblock Submodular maximization beyond non-negativity: Guarantees, fast
  algorithms, and applications.
\newblock In \emph{International Conference on Machine Learning}, pages
  2634--2643. PMLR, 2019.

\bibitem[Hassani et~al.(2017)Hassani, Soltanolkotabi, and
  Karbasi]{hassani2017gradient}
Hamed Hassani, Mahdi Soltanolkotabi, and Amin Karbasi.
\newblock Gradient methods for submodular maximization.
\newblock In \emph{Advances in Neural Information Processing Systems}, pages
  5841--5851, 2017.

\bibitem[Hassani et~al.(2020)Hassani, Karbasi, Mokhtari, and
  Shen]{hassani2020stochastic}
Hamed Hassani, Amin Karbasi, Aryan Mokhtari, and Zebang Shen.
\newblock Stochastic conditional gradient++:(non) convex minimization and
  continuous submodular maximization.
\newblock \emph{SIAM Journal on Optimization}, 30\penalty0 (4):\penalty0
  3315--3344, 2020.

\bibitem[Hazan et~al.(2016{\natexlab{a}})Hazan, Levy, and
  Shalev-Shwartz]{hazan2016graduated}
Elad Hazan, Kfir~Yehuda Levy, and Shai Shalev-Shwartz.
\newblock On graduated optimization for stochastic non-convex problems.
\newblock In \emph{International Conference on Machine Learning}, pages
  1833--1841. PMLR, 2016{\natexlab{a}}.

\bibitem[Hazan et~al.(2016{\natexlab{b}})]{hazan2019introduction}
Elad Hazan et~al.
\newblock Introduction to online convex optimization.
\newblock \emph{Foundations and Trends{\textregistered} in Optimization},
  2\penalty0 (3-4):\penalty0 157--325, 2016{\natexlab{b}}.

\bibitem[Kempe et~al.(2003)Kempe, Kleinberg, and Tardos]{kempe2003maximizing}
David Kempe, Jon Kleinberg, and {\'E}va Tardos.
\newblock Maximizing the spread of influence through a social network.
\newblock In \emph{Proceedings of the ninth ACM SIGKDD International Conference
  on Knowledge Discovery and Data Mining}, pages 137--146, 2003.

\bibitem[Khanna et~al.(1998)Khanna, Motwani, Sudan, and
  Vazirani]{khanna1998syntactic}
Sanjeev Khanna, Rajeev Motwani, Madhu Sudan, and Umesh Vazirani.
\newblock On syntactic versus computational views of approximability.
\newblock \emph{SIAM Journal on Computing}, 28\penalty0 (1):\penalty0 164--191,
  1998.

\bibitem[Leskovec et~al.(2007)Leskovec, Krause, Guestrin, Faloutsos,
  VanBriesen, and Glance]{leskovec2007cost}
Jure Leskovec, Andreas Krause, Carlos Guestrin, Christos Faloutsos, Jeanne
  VanBriesen, and Natalie Glance.
\newblock Cost-effective outbreak detection in networks.
\newblock In \emph{Proceedings of the 13th ACM SIGKDD International Conference
  on Knowledge Discovery and Data Mining}, pages 420--429, 2007.

\bibitem[Lin and Bilmes(2011)]{lin2011class}
Hui Lin and Jeff Bilmes.
\newblock A class of submodular functions for document summarization.
\newblock In \emph{Proceedings of the 49th Annual Meeting of the Association
  for Computational Linguistics: Human Language Technologies}, pages 510--520,
  2011.

\bibitem[Liu et~al.(2020)Liu, Deng, Li, Chen, and So]{liu2020nonconvex}
Huikang Liu, Zengde Deng, Xiao Li, Shixiang Chen, and Anthony Man-Cho So.
\newblock Nonconvex robust synchronization of rotations.
\newblock In \emph{NeurIPS Annual Workshop on Optimization for Machine
  Learning}, pages 1--7, 2020.

\bibitem[Lov{\'a}sz(1983)]{lovasz1983submodular}
L{\'a}szl{\'o} Lov{\'a}sz.
\newblock Submodular functions and convexity.
\newblock In \emph{Mathematical programming the state of the art}, pages
  235--257. Springer, 1983.

\bibitem[Mehta et~al.(2007)Mehta, Saberi, Vazirani, and
  Vazirani]{mehta2007adwords}
Aranyak Mehta, Amin Saberi, Umesh Vazirani, and Vijay Vazirani.
\newblock Adwords and generalized online matching.
\newblock \emph{Journal of the ACM}, 54\penalty0 (5):\penalty0 22--es, 2007.

\bibitem[Mitra et~al.(2021)Mitra, Feldman, and Karbasi]{mitra2021submodular+}
Siddharth Mitra, Moran Feldman, and Amin Karbasi.
\newblock Submodular+ concave.
\newblock In \emph{Advances in Neural Information Processing Systems}, 2021.

\bibitem[Mokhtari et~al.(2018)Mokhtari, Hassani, and
  Karbasi]{mokhtari2018conditional}
Aryan Mokhtari, Hamed Hassani, and Amin Karbasi.
\newblock Conditional gradient method for stochastic submodular maximization:
  Closing the gap.
\newblock In \emph{International Conference on Artificial Intelligence and
  Statistics}, pages 1886--1895. PMLR, 2018.

\bibitem[Murty and Kabadi(1987)]{murty1987some}
Katta~G Murty and Santosh~N Kabadi.
\newblock Some np-complete problems in quadratic and nonlinear programming.
\newblock \emph{Mathematical Programming}, 39\penalty0 (2):\penalty0 117--129,
  1987.

\bibitem[Nemhauser et~al.(1978)Nemhauser, Wolsey, and
  Fisher]{nemhauser1978analysis}
George~L Nemhauser, Laurence~A Wolsey, and Marshall~L Fisher.
\newblock An analysis of approximations for maximizing submodular set
  functions—i.
\newblock \emph{Mathematical Programming}, 14\penalty0 (1):\penalty0 265--294,
  1978.

\bibitem[Nesterov(2013)]{nesterov2013introductory}
Y~Nesterov.
\newblock \emph{Introductory Lectures on Convex Optimization: A Basic Course},
  volume~87.
\newblock Springer Science \& Business Media, 2013.

\bibitem[Netrapalli et~al.(2014)Netrapalli, U~N, Sanghavi, Anandkumar, and
  Jain]{NIPS2014_443cb001}
Praneeth Netrapalli, Niranjan U~N, Sujay Sanghavi, Animashree Anandkumar, and
  Prateek Jain.
\newblock Non-convex robust pca.
\newblock In \emph{Advances in Neural Information Processing Systems}, pages
  1107--1115, 2014.

\bibitem[Quanrud and Khashabi(2015)]{quanrud2015online}
Kent Quanrud and Daniel Khashabi.
\newblock Online learning with adversarial delays.
\newblock In \emph{Advances in Neural Information Processing Systems}, pages
  1270--1278, 2015.

\bibitem[Streeter and Golovin(2008)]{streeter2008online}
Matthew Streeter and Daniel Golovin.
\newblock An online algorithm for maximizing submodular functions.
\newblock In \emph{Advances in Neural Information Processing Systems}, pages
  1577--1584, 2008.

\bibitem[Yang et~al.(2016)Yang, Mao, Pei, and He]{yang2016continuous}
Yu~Yang, Xiangbo Mao, Jian Pei, and Xiaofei He.
\newblock Continuous influence maximization: What discounts should we offer to
  social network users?
\newblock In \emph{Proceedings of the 2016 International Conference on
  Management of Data}, pages 727--741, 2016.

\bibitem[Zhang et~al.(2019)Zhang, Chen, Hassani, and Karbasi]{zhang2019online}
Mingrui Zhang, Lin Chen, Hamed Hassani, and Amin Karbasi.
\newblock Online continuous submodular maximization: From full-information to
  bandit feedback.
\newblock In \emph{Advances in Neural Information Processing Systems}, pages
  9206--9217, 2019.

\bibitem[Zinkevich(2003)]{zinkevich2003online}
Martin Zinkevich.
\newblock Online convex programming and generalized infinitesimal gradient
  ascent.
\newblock In \emph{International Conference on Machine Learning}, pages
  928--936, 2003.

\end{thebibliography}

% \onecolumn
\appendix
\section{Proofs in Section~\ref{sec:non-oblivious}}\label{Appendix:A}
\subsection{Proof of \cref{lemma:2}} \label{proof:lem1}
First, we review some basic inequalities for $\gamma$-weakly continuous DR-submodular function $f$.
\begin{lemma}\label{lemma:a1}
For a monotone, differentiable, and $\gamma$-weakly continuous DR-submodular function $f$, we have
\begin{enumerate}
\item For any $\boldsymbol{x}\le\boldsymbol{y}$, we have $\langle\boldsymbol{y}-\boldsymbol{x}, \nabla f(\boldsymbol{x})\rangle\ge \gamma(f(\boldsymbol{y})-f(\boldsymbol{x}))$ and $\langle \boldsymbol{y}-\boldsymbol{x}, \nabla f(\boldsymbol{y})\rangle\le\frac{1}{\gamma}(f(\boldsymbol{y})-f(\boldsymbol{x}))$.
\item For any $\boldsymbol{x},\boldsymbol{y}\in\mathcal{X}$, we also could derive $ \langle \boldsymbol{y}-\boldsymbol{x}, \nabla f(\boldsymbol{x})\rangle\ge\gamma f(\boldsymbol{x}\lor \boldsymbol{y})+\frac{1}{\gamma}f(\boldsymbol{x}\land \boldsymbol{y} )-(\gamma+\frac{1}{\gamma})f(\boldsymbol{x})$.
\end{enumerate}\end{lemma}
\begin{proof} First, according to the definition of DR-submodular function and monotone property in \cref{sec:pre}, we have $\nabla f(\boldsymbol{x})\ge\gamma\nabla f(\boldsymbol{y})$, if $\boldsymbol{x}\le\boldsymbol{y}$. Thus, for any $\boldsymbol{x}\le\boldsymbol{y}$, we have
\begin{equation}\label{equ:11}
    \begin{aligned}
    &f(\boldsymbol{y})-f(\boldsymbol{x})=\int_{0}^{1}\langle \boldsymbol{y}-\boldsymbol{x}, \nabla f(\boldsymbol{x}+z(\boldsymbol{y}-\boldsymbol{x}))\rangle dz\le \frac{1}{\gamma}\langle \boldsymbol{y}-\boldsymbol{x}, \nabla f(\boldsymbol{x}))\rangle,\\
    &f(\boldsymbol{y})-f(\boldsymbol{x})=\int_{0}^{1}\langle \boldsymbol{y}-\boldsymbol{x}, \nabla f(\boldsymbol{x}+z(\boldsymbol{y}-\boldsymbol{x}))\rangle dz\ge \gamma\langle \boldsymbol{y}-\boldsymbol{x}, \nabla f(\boldsymbol{y})\rangle,
    \end{aligned} 
\end{equation} where these two inequalities follow from  $\boldsymbol{y}\ge\boldsymbol{x}+z(\boldsymbol{y}-\boldsymbol{x})\ge\boldsymbol{x}$ such that $\frac{1}{\gamma}\nabla f(\boldsymbol{x})\ge\nabla f(\boldsymbol{x}+z(\boldsymbol{y}-\boldsymbol{x}))\ge\gamma\nabla f(\boldsymbol{y})$ for any $z\in[0,1]$. We finish the proof of the first inequality in \cref{lemma:a1}.

Then, from \eqref{equ:11}, we could derive that
\begin{equation}\label{equ:12}
\begin{aligned}
 &\langle \boldsymbol{y}\lor \boldsymbol{x} -\boldsymbol{x}, \nabla f(\boldsymbol{x})\rangle \ge \gamma f(\boldsymbol{y}\lor \boldsymbol{x})-\gamma f(\boldsymbol{x}), \\
 &\langle \boldsymbol{x}\land \boldsymbol{y} -\boldsymbol{x}, \nabla f(\boldsymbol{x})\rangle \ge \frac{1}{\gamma} (f(\boldsymbol{x}\land\boldsymbol{y})-f(\boldsymbol{x})),
 \end{aligned}
 \end{equation} where $ \boldsymbol{y}\lor \boldsymbol{x}\ge\boldsymbol{x}$ and $\boldsymbol{x}\land \boldsymbol{y}\le\boldsymbol{x}$.

Merging the two equations in \eqref{equ:12}, we have, for any $\boldsymbol{x}$ and $\boldsymbol{y}\in\mathcal{X}$,
\begin{equation}\label{equ:13}
\begin{aligned}
\langle\boldsymbol{y}-\boldsymbol{x}, \nabla f(\boldsymbol{x})\rangle
  &= \langle \boldsymbol{y}\lor \boldsymbol{x} -\boldsymbol{x}, \nabla f(\boldsymbol{x})\rangle+\langle \boldsymbol{x}\land \boldsymbol{y} -\boldsymbol{x}, \nabla f(\boldsymbol{x})\rangle\\
  &\ge\gamma f(\boldsymbol{x}\lor \boldsymbol{y})+\frac{1}{\gamma}f(\boldsymbol{x}\land \boldsymbol{y} )-(\gamma+\frac{1}{\gamma})f(\boldsymbol{x}),  
\end{aligned}
\end{equation} 
where $\boldsymbol{x}\land \boldsymbol{y}+ \boldsymbol{x}\lor \boldsymbol{y}=\boldsymbol{x}+\boldsymbol{y}$.
Thus, we prove the second inequality in \cref{lemma:a1}. 
\end{proof}

Next, with the~\cref{lemma:a1}, we prove the~\cref{lemma:2}. 

\begin{proof}From \cref{equ:13}, if $\boldsymbol{x}$ is a stationary point of $f$ in domain $\mathcal{C}$, we have $(\gamma+\frac{1}{\gamma})f(\boldsymbol{x})\ge\gamma f(\boldsymbol{x}\lor \boldsymbol{y})+\frac{1}{\gamma}f(\boldsymbol{x}\land \boldsymbol{y} )  $ for any $\boldsymbol{y}\in\mathcal{C}$. Due to the monotone and non-negative property, $f(\boldsymbol{x})\ge\frac{\gamma^{2}}{\gamma^{2}+1}\max_{\boldsymbol{y}\in\mathcal{C}}f(\boldsymbol{y})$.
\end{proof}
\subsection{Proof of \cref{lemma:3}} \label{proof:lem2}
\begin{proof} First, we obtain an inequality about $\langle \boldsymbol{x},\nabla F(\boldsymbol{x})\rangle$, i.e., 
\begin{equation}\label{equ:appendix1}
    \begin{aligned}
    \langle \boldsymbol{x},\nabla F(\boldsymbol{x})\rangle&=\int_{0}^{1} w(z)\langle \boldsymbol{x},\nabla f(z*\boldsymbol{x})\rangle\mathrm{d}z\\ 
    &=\int^{1}_{0}w(z) \mathrm{d}f(z*\boldsymbol{x})\\
    &=w(z)f(z*\boldsymbol{x})|_{z=0}^{z=1}-\int_{0}^{1}f(z*\boldsymbol{x})w'(z)\mathrm{d}z\\
    &\le w(1)f(\boldsymbol{x})-\int_{0}^{1}f(z*\boldsymbol{x})w'(z)\mathrm{d}z.
    \end{aligned}
\end{equation}
Then, we also prove some properties about $\langle \boldsymbol{y},\nabla F(\boldsymbol{x})\rangle$, namely, 
\begin{equation}\label{equ:appendix2}
    \begin{aligned}
    \langle \boldsymbol{y},\nabla F(\boldsymbol{x})\rangle&=\int_{0}^{1} w(z)\langle \boldsymbol{y},\nabla f(z*\boldsymbol{x})\rangle\mathrm{d}z\\ 
    &\ge\int^{1}_{0}w(z)\langle \boldsymbol{y}\lor(z*\boldsymbol{x})-z*\boldsymbol{x},\nabla f(z*\boldsymbol{x})\rangle\mathrm{d}z\\
    &\ge\gamma\int^{1}_{0}w(z)(f(\boldsymbol{y}\lor(z*\boldsymbol{x}))-f(z*\boldsymbol{x}))\mathrm{d}z \\
    &\ge(\gamma\int_{0}^{1}w(z)dz)f(\boldsymbol{y})-\int^{1}_{0}\gamma w(z)f(z*\boldsymbol{x})\mathrm{d}z,
    \end{aligned}
\end{equation}
where the first inequality follows from $\boldsymbol{y}\ge\boldsymbol{y}\lor(z*\boldsymbol{x})-z*\boldsymbol{x}\ge\boldsymbol{0}$ and $\nabla f(z*\boldsymbol{x})\ge\boldsymbol{0}$; the second one comes from the \cref{lemma:2}; and the final inequality follows from $f(\boldsymbol{y}\lor(z*\boldsymbol{x}))\ge f(\boldsymbol{y})$.
    
Finally, putting above the inequality~\eqref{equ:appendix1} and inequality~\eqref{equ:appendix2} together, we have
\begin{equation}\label{equ:rev1}
 \begin{aligned}
    \langle\boldsymbol{y}-\boldsymbol{x}, \nabla F(\boldsymbol{x})\rangle&\ge(\gamma\int_{0}^{1}w(z)\mathrm{d}z)f(\boldsymbol{y})-w(1)f(\boldsymbol{x})+\int^{1}_{0}(w'(z)-\gamma w(z))f(z*\boldsymbol{x})\mathrm{d}z\\
    &=(\gamma\int_{0}^{1}w(z)\mathrm{d}z)(f(\boldsymbol{y})-\frac{w(1)+\int^{1}_{0} (\gamma w(z)-w'(z))\frac{f(z*\boldsymbol{x})}{f(\boldsymbol{x})}\mathrm{d}z}{\gamma\int_{0}^{1}w(z)\mathrm{d}z}f(\boldsymbol{x}))\\
    &=(\gamma\int_{0}^{1}w(z)\mathrm{d}z)(f(\boldsymbol{y})-\theta(w,f,\boldsymbol{x})f(\boldsymbol{x}))\\
    &\ge(\gamma\int_{0}^{1}w(z)\mathrm{d}z)(f(\boldsymbol{y})-\theta(w)f(\boldsymbol{x})),
    \end{aligned}
    \end{equation} where the final inequality follows from $\theta(w)=\max_{f,\boldsymbol{x}}\theta(w,f,\boldsymbol{x})$.
\end{proof}

\subsection{Proof of \cref{thm:1}} \label{proof:thm1}
\begin{proof}
In this proof, we investigate the optimal value and solution about the following optimization problem:
\begin{equation}\label{equ:19}
   \begin{aligned}
    \min_{w}\theta(w)=\min_{w}\max_{f,\boldsymbol{x}}& \frac{w(1)+\int^{1}_{0} (\gamma w(z)-w'(z))\frac{f(z*\boldsymbol{x})}{f(\boldsymbol{x})}\mathrm{d}z}{\gamma\int_{0}^{1}w(z)\mathrm{d}z}\\
          \rm{s.t.} \ &w(z)\ge 0,\\
          &w(z)\in C^{1}[0,1],\\
          &f(\boldsymbol{x})>0, \\
          &\nabla f(\boldsymbol{x}_1)\ge\gamma\nabla f(\boldsymbol{y}_1)\ge\boldsymbol{0},  \forall \boldsymbol{x}_1\le \boldsymbol{y}_1.
    \end{aligned} 
\end{equation}

(1) Before going into the detail, we first consider a new optimization problem as follows:
\begin{equation}\label{equ:20}
    \begin{aligned}
    \min_{w}\max_{R}\ \ &\theta(w,R)\\
          \rm{s.t.} \ &w(z)\ge 0,\\
          &w(z)\in C^{1}[0,1],\\
          &\gamma\int_{0}^{1}w(z)dz=1,\\
          &R(z)\ge 0, \\
          &R(1)=1,\\
          & R'(z_{1})\ge\gamma R'(z_{2})\ge 0\ (\forall z_{1}\le z_{2}, z_{1}, z_{2}\in[0,1]),
    \end{aligned} 
\end{equation}where $\theta(w,R)=w(1)+\int^{1}_{0} (\gamma w(z)-w'(z))R(z)\mathrm{d}z$.

Next, we prove the equivalence between problem~\eqref{equ:19} and problem~\eqref{equ:20}. For any fixed point $\boldsymbol{x}\in\mathcal{C}$, we consider the function $m(z)=\frac{f(z*\boldsymbol{x})}{f(\boldsymbol{x})}$ (we assume $f(\boldsymbol{x})>0$), which is satisfied with the constraints of problem~\eqref{equ:20}, i.e., $m(z)\ge 0$, $m(1)=1$, and $m'(z_{1})=\frac{\langle\boldsymbol{x}, \nabla f(z_{1}*\boldsymbol{x})\rangle}{f(\boldsymbol{x})}\ge\frac{\gamma\langle\boldsymbol{x}, \nabla f(z_{2}*\boldsymbol{x})\rangle}{f(\boldsymbol{x})}=\gamma m'(z_{2})\ge 0$ $(\forall z_{1}\le z_{2}, z_{1}, z_{2}\in[0,1])$. Therefore, the optimal objective value of problem~\eqref{equ:20} is larger than that of problem~\eqref{equ:19}. Moreover, for any $R(z)$ satisfying the constrains in problem~\eqref{equ:20}, we can design a function $f_{1}(\boldsymbol{x})=R(x_{1}/a_{1})$, where $x_{1}$ (we assume $x_{1}\in[0,a_{1}]$ in the~\cref{sec:pre}) is the first coordinate of point $\boldsymbol{x}$. Also, $f_{1}(\boldsymbol{x})\ge0$ and when $\boldsymbol{x}\le\boldsymbol{y}$, we have $\nabla f_{1}(\boldsymbol{x})\ge\gamma\nabla f_{1}(\boldsymbol{y})$. Hence, $f_{1}$ is also satisfied with the constraints of problem~\eqref{equ:19}. If we set $\boldsymbol{x}_{1}=(a_{1},0,\dots,0)\in\mathcal{X}$, $\frac{f_{1}(z*\boldsymbol{x}_{1})}{f_{1}(\boldsymbol{x}_{1})}=R(z)$ such that the optimal objective value of problem~\eqref{equ:19} is larger than that of problem~\eqref{equ:20}. As a result, the optimization problem~\eqref{equ:20} is equivalent to the problem~\eqref{equ:19}.

(2) Then, we prove the $\min_{w}\max_{f,\boldsymbol{x}}\theta(w,f,\boldsymbol{x})\ge\frac{1}{1-e^{-\gamma}}$. Setting $\widehat{R}(z)=\frac{1-e^{-\gamma z}}{1-e^{-\gamma}}$, we could verify that, if $\gamma\int_{0}^{1}w(z)dz=1$, 
\begin{equation}\label{equ:21}
    \begin{aligned}
    \theta(w,\widehat{R})&=w(1)+\int^{1}_{0} (\gamma w(z)-w'(z))\widehat{R}(z)\mathrm{d}z\\
    &=w(1)+\frac{\int^{1}_{0} (\gamma w(z)-w'(z))\mathrm{d}z+\int^{1}_{0} e^{-\gamma z}(w'(z)-\gamma w(z))\mathrm{d}z}{1-e^{-\gamma}}\\
    &=w(1)+\frac{1-w(1)+w(0)+e^{-\gamma z}w(z)|_{z=0}^{z=1}}{1-e^{-\gamma}}\\
    &=w(1)+\frac{1-w(1)+w(0)+e^{-\gamma}w(1)-w(0)}{1-e^{-\gamma}}\\
    &=\frac{1}{1-e^{-\gamma}}.
    \end{aligned}
\end{equation}
Also, $\widehat{R}$ is satisfied with the constraints of optimization problem~\eqref{equ:20}, i.e., for any $z\in[0,1]$, $\widehat{R}(z)\ge0$, $\widehat{R}(1)=1$ and $\widehat{R}'(x)=\frac{\gamma e^{-\gamma x}}{1-e^{-\gamma}}\ge\frac{\gamma^{2} e^{-\gamma y}}{1-e^{-\gamma}}=\gamma\widehat{R}'(y)$ where $x\le y$ and $0\le\gamma\le 1$. Therefore, $\max_{R}\theta(w,R)\ge\theta(w,\widehat{R})=\frac{1}{1-e^{-\gamma}}$ and $\min_{w}\max_{f,\boldsymbol{x}}\theta(w,f,\boldsymbol{x})=\min_{w}\max_{R}\theta(w,R)\ge\frac{1}{1-e^{-\gamma}}$.

(3) We consider $\widehat{w}(z)=e^{\gamma(z-1)}$ and observe that $\widehat{w}'(z)=\gamma\widehat{w}(z)$ such that $\theta(\widehat{w},f,\boldsymbol{x})=\frac{\widehat{w}(1)+\int^{1}_{0} (\gamma \widehat{w}(z)-\widehat{w}'(z))\frac{f(z*\boldsymbol{x})}{f(\boldsymbol{x})}\mathrm{d}z}{\gamma\int_{0}^{1}\widehat{w}(z)\mathrm{d}z}=\frac{\widehat{w}(1)}{\gamma\int_{0}^{1}\widehat{w}(z)\mathrm{d}z}=\frac{1}{1-e^{-\gamma}}$ for any function $f$. Also, $\widehat{w}(z)$ is satisfied with the constraints in optimization problem~\eqref{equ:19}, namely, $\widehat{w}(z)\ge 0$ and $\widehat{w}\in C^{1}[0,1]$. Therefore, $\frac{1}{1-e^{-\gamma}}=\min_{w}\max_{f,\boldsymbol{x}}\theta(w,f,\boldsymbol{x})$ and $e^{\gamma(z-1)}\in\arg\min_{w}\theta(w)$.
\end{proof}

\subsection{Proof of \cref{thm:2}} \label{proof:thm2}
\begin{proof} From the definition of $F$, we have $\langle\boldsymbol{y}-\boldsymbol{x}, \nabla F(\boldsymbol{x})\rangle \ge (1-e^{-\gamma})f(\boldsymbol{y})-f(\boldsymbol{x})$ for any point $\boldsymbol{x}, \boldsymbol{y}\in\mathcal{C}$. Hence, when $\boldsymbol{x}\in\mathcal{C}$ is a stationary point for $F$ in the domain $\mathcal{C}$,
$0\ge\langle\boldsymbol{y}-\boldsymbol{x}, \nabla F(\boldsymbol{x})\rangle \ge (1-e^{-\gamma})f(\boldsymbol{y})-f(\boldsymbol{x})$ for any point $\boldsymbol{y}\in\mathcal{C}$ such that $f(\boldsymbol{x})\ge(1-e^{-\gamma})\max_{\boldsymbol{y}\in\mathcal{C}}f(\boldsymbol{y})$. 

Then, for the second one, we first verify that the value $\int_{0}^{1}\frac{e^{\gamma(z-1)}}{z}f(z*\boldsymbol{x})\mathrm{d}z$ is controlled via $f(\boldsymbol{x})$ for any $\boldsymbol{x}\in\mathcal{X}$. For any $\delta\in(0,1)$,  we first have
\begin{equation}\label{equ:22}
\begin{aligned}
 &\int_{0}^{1}\frac{e^{\gamma(z-1)}}{z}f(z*\boldsymbol{x})\mathrm{d}z\\
 &=(\int_{0}^{\delta}+\int_{\delta}^{1}) \frac{e^{\gamma(z-1)}}{z}f(z*\boldsymbol{x})\mathrm{d}z\\
 &\le\int_{0}^{\delta}\frac{f(z*x)}{z}\mathrm{d}z+(\int_{\delta}^{1}\frac{1}{z}\mathrm{d}z)f(\boldsymbol{x})\\
 &= \int_{0}^{\delta}\frac{f(z*x)}{z}\mathrm{d}z+\ln(\frac{1}{\delta})f(\boldsymbol{x})\\
 &= \int_{0}^{\delta}\frac{\int_{0}^{z}\langle \boldsymbol{x},\nabla f(u*\boldsymbol{x})\rangle \mathrm{d}u}{z}\mathrm{d}z+\ln(\frac{1}{\delta})f(\boldsymbol{x}),
 \end{aligned}
\end{equation} where the first inequality follows from $f(z*\boldsymbol{x})\le f(\boldsymbol{x})$ and $\delta\in[0,1]$, and the final equality from $\int_{0}^{z}\langle \boldsymbol{x},\nabla f(u*\boldsymbol{x})\rangle \mathrm{d}u=f(z*\boldsymbol{x})-f(\boldsymbol{0})=f(z*\boldsymbol{x})$. 

Next, \begin{equation}\label{equ:23}
    \begin{aligned}
     \int_{0}^{\delta}\frac{\int_{0}^{z}\langle \boldsymbol{x},\nabla f(u*\boldsymbol{x})\rangle \mathrm{d}u}{z}\mathrm{d}z&=\int_{0}^{\delta}\langle \boldsymbol{x},\nabla f(u*\boldsymbol{x})\rangle \int_{u}^{\delta}\frac{1}{z}\mathrm{d}z\mathrm{d}u\\
 &=\int_{0}^{\delta}\langle \boldsymbol{x},\nabla f(u*\boldsymbol{x})\rangle \ln(\frac{\delta}{u})\mathrm{d}u\\
 &=\int_{0}^{\delta}(\langle \boldsymbol{x},\nabla f(u*\boldsymbol{x})-\nabla f(\boldsymbol{x})\rangle+\langle \boldsymbol{x}, \nabla f(\boldsymbol{x})\rangle)\ln(\frac{\delta}{u})\mathrm{d}u\\
 &\le\int_{0}^{\delta}\ln(\frac{\delta}{u})\mathrm{d}u(L r^2(\mathcal{X})+\frac{f(\boldsymbol{x})}{\gamma})\\
  &=(u-u\ln(\frac{u}{\delta}))|_{u=0}^{\delta}(L r^2(\mathcal{X})+\frac{f(\boldsymbol{x})}{\gamma})\\
  &=\delta(L r^2(\mathcal{X})+\frac{f(\boldsymbol{x})}{\gamma}),
    \end{aligned}
\end{equation} where the first equality follows from the Fubini's theorem; in the first inequality, we use  $\langle \boldsymbol{x},\nabla f(u*\boldsymbol{x})-\nabla f(\boldsymbol{x})\rangle\le L\left\|\boldsymbol{x}\right\|^{2}$, which is derived from the $L$-smooth property, and $\langle \boldsymbol{x}, \nabla f(\boldsymbol{x})\rangle\le\frac{  f(\boldsymbol{x})}{\gamma}$, following from the \cref{lemma:2} and $f(\boldsymbol{0})=0$; the final equality follows from $\lim_{u\rightarrow 0_{+}} u\ln(u)=0$.

From \cref{equ:22} and \cref{equ:23}, for any $\delta\in(0,1)$, we have 
\begin{equation}\label{equ:24}
    \begin{aligned}
     F(\boldsymbol{x})&\le \ln(\frac{1}{\delta})f(\boldsymbol{x})+\delta(L_{*}r^2(\mathcal{X})+\frac{f(\boldsymbol{x})}{\gamma})\\
     &\le\ln(\frac{1}{\delta})(f(\boldsymbol{x})+c)+\delta(L_{*}r^2(\mathcal{X})+\frac{f(\boldsymbol{x})}{\gamma}),
    \end{aligned}
\end{equation} where the second inequality comes from $c>0$.

If we set $\delta=\frac{f(\boldsymbol{x})+c}{\frac{f(\boldsymbol{x})}{\gamma}+L r^2(\mathcal{X})}\in[0,1]$ ($0\le\gamma\le 1$ and $0<c\le L_{*}r^{2}(\mathcal{X})$), we have 
\begin{align*}
     F(\boldsymbol{x})&\le\ln(\frac{1}{\delta})(f(\boldsymbol{x})+c)+\delta(L_{*}r^2(\mathcal{X})+\frac{f(\boldsymbol{x})}{\gamma})\\
     &=(1+\ln(\frac{1}{\delta}))(f(\boldsymbol{x})+c)\\
     &\le(1+\ln(\tau)(f(\boldsymbol{x})+c),
\end{align*} where the final inequality is derived from $\frac{1}{\delta}\le\tau$ and $\tau=max(\frac{1}{\gamma},\frac{L_{*}r^{2}(\mathcal{X})}{c})$.

As a result, the value $\int_{0}^{1}\frac{e^{\gamma(z-1)}}{z}f(z*\boldsymbol{x})\mathrm{d}z$ is well-defined. We also could verify that $\nabla \int_{0}^{1}\frac{e^{\gamma(z-1)}}{z}f(z*\boldsymbol{x})\mathrm{d}z=\int_{0}^{1}e^{\gamma(z-1)}\nabla f(z*\boldsymbol{x})\mathrm{d}z$ so that we could set $F(\boldsymbol{x})=\int_{0}^{1}\frac{e^{\gamma(z-1)}}{z}f(z*\boldsymbol{x})\mathrm{d}z$.

For the final one, 
\begin{equation}\label{equ:25}
    \begin{aligned}
    \left\|\nabla F(\boldsymbol{x})-\nabla F(\boldsymbol{y})\right\|&=\left\|\int_{0}^{1}e^{\gamma(z-1)}(\nabla f(z*\boldsymbol{x})-\nabla f(z*\boldsymbol{y}))\mathrm{d}z\right\|\\
    &\le\int_{0}^{1}e^{\gamma(z-1)}\left\|\nabla f(z*\boldsymbol{x})-\nabla f(z*\boldsymbol{y})\right\|\mathrm{d}z\\
    &\le L(\int_{0}^{1}e^{\gamma(z-1)}z dz)\left\|\boldsymbol{x}-\boldsymbol{y}\right\|\\
    &=\frac{\gamma+e^{-\gamma}-1}{\gamma^2}L\left\|\boldsymbol{x}-\boldsymbol{y}\right\|.
    \end{aligned}
\end{equation}
\end{proof}
\subsection{Proof of \cref{prop:1}} \label{proof:prop1}
\begin{proof}
For the first one, fixed $z$, $\mathbb{E}\left(\left.\widetilde{\nabla}f(z*\boldsymbol{x})\right|\boldsymbol{x},z\right)=\nabla f(z*\boldsymbol{x})$ such that $\mathbb{E}\left(\left.\widetilde{\nabla}f(z*\boldsymbol{x})\right|\boldsymbol{x}\right)=\mathbb{E}_{z\sim\mathbf{Z}}\left(\mathbb{E}\left(\left.\widetilde{\nabla}f(z*\boldsymbol{x})\right|\boldsymbol{x},z
\right)\right)=\mathbb{E}_{z\sim\mathbf{Z}}\left(\left.\nabla f(z*\boldsymbol{x})\right|\boldsymbol{x}\right)=\int_{z=0}^{1} \frac{\gamma e^{\gamma(z-1)}}{1-e^{-\gamma}}\nabla f(z*\boldsymbol{x})\mathrm{d}z=\frac{\gamma}{1-e^{-\gamma}}F(\boldsymbol{x})$. For the second one, 
        \begin{align*}
        &\mathbb{E}\left(\left.\left\|\frac{1-e^{-\gamma}}{\gamma}\widetilde{\nabla}f(z*\boldsymbol{x})-\nabla F(\boldsymbol{x})\right\|^{2}\right|\boldsymbol{x}\right)\\
        =&\mathbb{E}\left(\left.\left\|\frac{1-e^{-\gamma}}{\gamma}(\widetilde{\nabla}f(z*\boldsymbol{x})-\nabla f(z*\boldsymbol{x}))+\frac{1-e^{-\gamma}}{\gamma}\nabla f(z*\boldsymbol{x})-\nabla F(\boldsymbol{x})\right\|^{2}\right|\boldsymbol{x}\right)\\
        \le & 2\mathbb{E}_{z\sim\mathbf{Z}}\left(\mathbb{E}\left(\left.\left\|\frac{1-e^{-\gamma}}{\gamma}(\widetilde{\nabla}f(z*\boldsymbol{x})-\nabla f(z*\boldsymbol{x}))\right\|^{2}\right|\boldsymbol{x},z\right) +\left\|\frac{1-e^{-\gamma}}{\gamma}\nabla f(z*\boldsymbol{x})-\nabla F(\boldsymbol{x})\right\|^{2}\right)\\
        \le & 2\frac{(1-e^{-\gamma})^{2}\sigma^{2}}{\gamma^{2}}+2\mathbb{E}_{z\sim\mathbf{Z}}\left(\left.\left\|\frac{1-e^{-\gamma}}{\gamma}\nabla f(z*\boldsymbol{x})-\nabla F(\boldsymbol{x})\right\|^{2}\right|\boldsymbol{x}\right) \\
        \le & 2\frac{(1-e^{-\gamma})^{2}\sigma^{2}}{\gamma^{2}}+2\mathbb{E}_{z\sim\mathbf{Z}}\left(\left.\left\|\int_{0}^{1} e^{\gamma(u-1)}(\nabla f(z*\boldsymbol{x})-\nabla f(u*\boldsymbol{x}))\mathrm{d}u\right
        \|^{2}\right|\boldsymbol{x}\right)\\
        \le & 2\frac{(1-e^{-\gamma})^{2}\sigma^{2}}{\gamma^{2}}+2\mathbb{E}_{z\sim\mathbf{Z}}\left(\left.\left(\int_{0}^{1}e^{\gamma(u-1)}|z-u|L\left\|\boldsymbol{x}\right
        \|\mathrm{d}u\right)^{2}\right|\boldsymbol{x}\right)\\
        \le & 2\frac{(1-e^{-\gamma})^{2}\sigma^{2}}{\gamma^{2}}+2\mathbb{E}_{z\sim\mathbf{Z}}\left(\left.\int_{0}^{1}e^{\gamma(u-1)}\mathrm{d}u\int_{u=0}^{1}e^{\gamma(u-1)}(z-u)^2L^{2}\left\|\boldsymbol{x}\right
        \|^{2}\mathrm{d}u\right|\boldsymbol{x}\right)\\
        = & 2\frac{(1-e^{-\gamma})^{2}\sigma^{2}}{\gamma^{2}}+2 \int_{z=0}^{1}\int_{u=0}^{1}e^{\gamma(u+z-2)}(z-u)^{2}L^{2}\left\|\boldsymbol{x}\right
        \|^{2} \mathrm{d}u\mathrm{d}z\\
        \le & 2\frac{(1-e^{-\gamma})^{2}\sigma^{2}}{\gamma^{2}}+\frac{2L^{2}r^{2}(\mathcal{X})(1-e^{-2\gamma})}{3\gamma},
        \end{align*}
     where the first and fifth inequalities come from Cauchy–Schwarz inequality.
    \end{proof}
\section{Proof of Theorem~\ref{thm:4}}\label{Appendix:B}
First, we recall the projection theorem from \citep{bertsekas2015convex} in the following lemma. 
\begin{lemma}\label{lemma:9}
For the projection $\mathcal{P}_{\mathcal{C}}(\boldsymbol{x})=\arg\min_{\boldsymbol{z}\in\mathcal{C}}\left\|\boldsymbol{z}-\boldsymbol{x}\right\|$,  we have
\begin{equation}\label{equ:41}
    \begin{aligned}
      \langle\mathcal{P}_{\mathcal{C}}(\boldsymbol{x}))-\boldsymbol{x},\boldsymbol{z}-\mathcal{P}_{\mathcal{C}}(\boldsymbol{x})\rangle\ge 0, \forall \boldsymbol{z}\in\mathcal{C}.
    \end{aligned}
\end{equation}
\end{lemma}
Before verifying the \cref{thm:4}, we first provide following lemma. 
\begin{lemma} \label{lemma:10} In the $t$-round update in ~\cref{alg:1},
if we set the $\widetilde{\nabla} F(\boldsymbol{x}_t)=\frac{1-e^{-\gamma}}{\gamma}\widetilde{\nabla}f(z_{t}*\boldsymbol{x}_{t})$, for any $\boldsymbol{y}\in\mathcal{C}$ and $\mu_{t}>0$, we have 
\begin{align*}
    &\mathbb{E}\left(F(\boldsymbol{x}_{t+1})-F(\boldsymbol{x}_{t})+f(\boldsymbol{x}_{t})-(1-e^{-\gamma})f(\boldsymbol{y})\right)\\
    \ge & \mathbb{E}\left(\frac{1}{2\eta_{t}}(\left\|\boldsymbol{y}-\boldsymbol{x}_{t+1}\right\|^{2}-\left\|\boldsymbol{y}-\boldsymbol{x}_{t}\right\|^{2})-\frac{1}{2\mu_{t}}\left\|\nabla F(\boldsymbol{x}_{t})-\widetilde{\nabla} F(\boldsymbol{x}_t) \right\|^{2}+(\frac{1}{2\eta_{t}}-\frac{\mu_{t}+L_{\gamma}}{2})\left\|\boldsymbol{x}_{t+1}-\boldsymbol{x}_{t})\right\|^{2}\right).
\end{align*} %where $L_{\gamma}=L\frac{\gamma+e^{-\gamma}-1}{\gamma^{2}}$.

\end{lemma}
\begin{proof} 
From the \cref{thm:2}, when $f$ is $L$-$smooth$, the non-oblivious function $F$ is $L_{\gamma}$-$smooth$. Hence
\begin{equation}\label{equ:43}
    \begin{aligned}
    F(\boldsymbol{x}_{t+1})-F(\boldsymbol{x}_{t})&=\int_{z=0}^{z=1}\langle \boldsymbol{x}_{t+1}-\boldsymbol{x}_{t}, \nabla F(\boldsymbol{x}_{t}+z(\boldsymbol{x}_{t+1}-\boldsymbol{x}_{t}))\rangle dz \\
    &\ge \langle \boldsymbol{x}_{t+1}-\boldsymbol{x}_{t}, \nabla F(\boldsymbol{x}_t) \rangle-\frac{L_{\gamma}}{2}\left\|\boldsymbol{x}_{t+1}-\boldsymbol{x}_{t}\right\|^{2}.\\
    \end{aligned}
\end{equation} 
 
Then,
\begin{equation}\label{equ:44}
 \begin{aligned}
 &\langle \boldsymbol{x}_{t+1}-\boldsymbol{x}_{t}, \nabla F(\boldsymbol{x}_t) \rangle\\
= &\langle \boldsymbol{x}_{t+1}-\boldsymbol{x}_{t}, \widetilde{\nabla} F(\boldsymbol{x}_t) \rangle+\langle \boldsymbol{x}_{t+1}-\boldsymbol{x}_{t}, \nabla F(\boldsymbol{x}_{t})-\widetilde{\nabla} F(\boldsymbol{x}_t) \rangle\\
\ge & \langle \boldsymbol{x}_{t+1}-\boldsymbol{x}_{t}, \widetilde{\nabla} F(\boldsymbol{x}_t) \rangle-\frac{1}{2\mu_{t}}\left\|\nabla F(\boldsymbol{x}_{t})-\widetilde{\nabla} F(\boldsymbol{x}_t) \right\|^{2}-\frac{\mu_{t}}{2}\left\|\boldsymbol{x}_{t+1}-\boldsymbol{x}_{t}\right\|^{2},
\end{aligned}
\end{equation} where the first inequality from the Young's inequality.
   
It is well known $\widetilde{\nabla}F(\boldsymbol{x}_{t})=\frac{1}{\eta_{t}}(\boldsymbol{y}_{t+1}-\boldsymbol{x}_{t})$ such that
\begin{equation}\label{equ:45}
   \begin{aligned}
    &\langle \boldsymbol{x}_{t+1}-\boldsymbol{x}_{t}, \widetilde{\nabla} F(\boldsymbol{x}_t) \rangle\\
    = & \langle\boldsymbol{x}_{t+1}-\boldsymbol{y}, \widetilde{\nabla} F(\boldsymbol{x}_t) \rangle+\langle \boldsymbol{y}-\boldsymbol{x}_{t}, \widetilde{\nabla} F(\boldsymbol{x}_t) \rangle\\
    = &\frac{1}{\eta_{t}}\langle \boldsymbol{x}_{t+1}-\boldsymbol{y}, \boldsymbol{y}_{t+1}-\boldsymbol{x}_{t}\rangle+\langle \boldsymbol{y}-\boldsymbol{x}_{t}, \widetilde{\nabla} F(\boldsymbol{x}_t) \rangle\\
    = &\frac{1}{\eta_{t}}\langle \boldsymbol{x}_{t+1}-\boldsymbol{y}, \boldsymbol{y}_{t+1}-\boldsymbol{x}_{t+1} \rangle+\frac{1}{\eta_{t}}\langle \boldsymbol{x}_{t+1}-\boldsymbol{y}, \boldsymbol{x}_{t+1}-\boldsymbol{x}_{t} \rangle+\langle \boldsymbol{y}-\boldsymbol{x}_{t}, \widetilde{\nabla} F(\boldsymbol{x}_t) \rangle\\
    \ge & \frac{1}{\eta_{t}}\langle\boldsymbol{x}_{t+1}-\boldsymbol{y}, \boldsymbol{x}_{t+1}-\boldsymbol{x}_{t} \rangle+\langle \boldsymbol{y}-\boldsymbol{x}_{t}, \widetilde{\nabla} F(\boldsymbol{x}_t) \rangle\\
    = & \frac{1}{2\eta_{t}}(\left\|\boldsymbol{y}-\boldsymbol{x}_{t+1}\right\|^{2}+\left\|\boldsymbol{x}_{t+1}-\boldsymbol{x}_{t}\right\|^{2}-\left\|\boldsymbol{y}-\boldsymbol{x}_{t}\right\|^{2})+\langle \boldsymbol{y}-\boldsymbol{x}_{t}, \widetilde{\nabla} F(\boldsymbol{x}_t) \rangle,
\end{aligned}    
\end{equation} where the first inequality follows from the ~\cref{lemma:9}.

From the Equation~\eqref{equ:43}-\eqref{equ:45}, we have
\begin{equation}\label{equ:46}
    \begin{aligned}
    &F(\boldsymbol{x}_{t+1})-F(\boldsymbol{x}_{t})\\
    \ge & \frac{1}{2\eta_{t}}(\left\|\boldsymbol{y}-\boldsymbol{x}_{t+1}\right\|^{2}+\left\|\boldsymbol{x}_{t+1}-\boldsymbol{x}_{t}\right\|^{2}-\left\|\boldsymbol{y}-\boldsymbol{x}_{t}\right\|^{2})+\langle \boldsymbol{y}-\boldsymbol{x}_{t}, \widetilde{\nabla} F(\boldsymbol{x}_t) \rangle\\
    &-\frac{1}{2\mu_{t}}\left\|\nabla F(\boldsymbol{x}_{t})-\widetilde{\nabla} F(\boldsymbol{x}_t) \right\|^{2}-\frac{\mu_{t}+L_{\gamma}}{2}\left\|\boldsymbol{x}_{t+1}-\boldsymbol{x}_{t}\right\|^{2}\\
    \ge &\frac{1}{2\eta_{t}}(\left\|\boldsymbol{y}-\boldsymbol{x}_{t+1}\right\|^{2}-\left\|\boldsymbol{y}-\boldsymbol{x}_{t}\right\|^{2})+\langle \boldsymbol{y}-\boldsymbol{x}_{t}, \widetilde{\nabla} F(\boldsymbol{x}_t) \rangle\\
    &-\frac{1}{2\mu_{t}}\left\|\nabla F(\boldsymbol{x}_{t})-\widetilde{\nabla} F(\boldsymbol{x}_t) \right\|^{2}+(\frac{1}{2\eta_{t}}-\frac{\mu_{t}+L_{\gamma}}{2})\left\|\boldsymbol{x}_{t+1}-\boldsymbol{x}_{t}\right\|^{2}.
    \end{aligned}
\end{equation} 

From the \cref{prop:1}, $\mathbb{E}(\widetilde{\nabla}F(\boldsymbol{x}_{t})|\boldsymbol{x}_{t})=\nabla F(\boldsymbol{x}_{t})$ and we also have
\begin{equation}\label{equ:47}
    \begin{aligned}
      &\mathbb{E}\left(F(\boldsymbol{x}_{t+1})-F(\boldsymbol{x}_{t})\right)\\
      \ge 
      &\mathbb{E}(\frac{1}{2\eta_{t}}(\left\|\boldsymbol{y}-\boldsymbol{x}_{t+1}\right\|^{2}-\left\|\boldsymbol{y}-\boldsymbol{x}_{t}\right\|^{2})+\mathbb{E}(\langle \boldsymbol{y}-\boldsymbol{x}_{t}, \widetilde{\nabla} F(\boldsymbol{x}_t) \rangle|\boldsymbol{x}_{t})\\
      &-\frac{1}{2\mu_{t}}\left\|\nabla F(\boldsymbol{x}_{t})-\widetilde{\nabla} F(\boldsymbol{x}_t) \right\|^{2}+(\frac{1}{2\eta_{t}}-\frac{\mu_{t}+L_{\gamma}}{2})\left\|\boldsymbol{x}_{t+1}-\boldsymbol{x}_{t}\right\|^{2})\\
      =&\mathbb{E}(\frac{1}{2\eta_{t}}(\left\|\boldsymbol{y}-\boldsymbol{x}_{t+1}\right\|^{2}-\left\|\boldsymbol{y}-\boldsymbol{x}_{t}\right\|^{2})+\langle \boldsymbol{y}-\boldsymbol{x}_{t}, \nabla F(\boldsymbol{x}_t) \rangle\\
      &-\frac{1}{2\mu_{t}}\left\|\nabla F(\boldsymbol{x}_{t})-\widetilde{\nabla} F(\boldsymbol{x}_t) \right\|^{2}+(\frac{1}{2\eta_{t}}-\frac{\mu_{t}+L_{\gamma}}{2})\left\|\boldsymbol{x}_{t+1}-\boldsymbol{x}_{t}\right\|^{2})\\
      \ge&\mathbb{E}(\frac{1}{2\eta_{t}}(\left\|\boldsymbol{y}-\boldsymbol{x}_{t+1}\right\|^{2}-\left\|\boldsymbol{y}-\boldsymbol{x}_{t}\right\|^{2})+(1-e^{-\gamma})f(\boldsymbol{y})-f(\boldsymbol{x}_{t})\\
      &-\frac{1}{2\mu_{t}}\left\|\nabla F(\boldsymbol{x}_{t})-\widetilde{\nabla} F(\boldsymbol{x}_t) \right\|^{2}+(\frac{1}{2\eta_{t}}-\frac{\mu_{t}+L_{\gamma}}{2})\left\|\boldsymbol{x}_{t+1}-\boldsymbol{x}_{t}\right\|^{2}),
    \end{aligned}
\end{equation} where the final inequality from the definition of $F$.
\end{proof}
Next, we prove the \cref{thm:4}.
\begin{proof}
From the Lemma~\ref{lemma:10}, if we set $\boldsymbol{y}=\boldsymbol{x}^{*}=\arg\max_{\boldsymbol{x}\in\mathcal{C}}f(\boldsymbol{x})$, we have 
 \begin{equation}\label{equ:48}
     \begin{aligned}
        &\sum_{t=1}^{T-1}\mathbb{E}(F(\boldsymbol{x}_{t+1})-F(\boldsymbol{x}_{t})+f(\boldsymbol{x}_{t})-(1-e^{-\gamma})f(\boldsymbol{x}^{*}))\\
        \ge&\sum_{t=1}^{T-1}\mathbb{E}(\frac{1}{2\eta_{t}}(\left\|\boldsymbol{x}^{*}-\boldsymbol{x}_{t+1}\right\|^{2}-\left\|\boldsymbol{x}^{*}-\boldsymbol{x}_{t}\right\|^{2})-\sum_{t=1}^{T-1}\frac{1}{2\mu_{t}}\left\|\nabla F(\boldsymbol{x}_{t})-\widetilde{\nabla} F(\boldsymbol{x}_t) \right\|^{2})\\
        \ge&-\sigma_{\gamma}^2\sum_{t=1}^{T}\frac{1}{2\mu_{t}}+\sum_{t=2}^{T-1}\mathbb{E}(\left\|\boldsymbol{x}^{*}-\boldsymbol{x}_{t}\right\|^{2})(\frac{1}{2\eta_{t-1}}-\frac{1}{2\eta_{t}})+\mathbb{E}(\frac{\left\|\boldsymbol{x}^{*}-\boldsymbol{x}_{T}\right\|^{2}}{2\eta_{T-1}}-\frac{\left\|\boldsymbol{x}^{*}-\boldsymbol{x}_{1}\right\|^{2}}{2\eta_{1}})\\
        \ge& -\frac{\mathrm{diam}^{2}(\mathcal{C})}{2\eta_{T-1}}-\sigma_{\gamma}^2\sum_{t=1}^{T}\frac{1}{\mu_{t}}\\
        \ge& -(\mathrm{diam}^{2}(\mathcal{C})L_{\gamma}/2+3\sigma_{\gamma}\mathrm{diam}(\mathcal{C})\sqrt{T}/2)
        \end{aligned}
 \end{equation} 
 where the first inequality follows from $\eta_{t}=\frac{1}{\mu_{t}+L_{r}}$ if we set $\mu_{t}=\frac{\sigma_{\gamma}\sqrt{t}}{\mathrm{diam}(\mathcal{C})}$ in \cref{lemma:10}; the second inequality from the \cref{prop:1} and the Abel's inequality; the third inequality from the definition of $\mathrm{diam}(\mathcal{C})$.

 Finally, we have:
 \begin{equation}\label{equ:49}
     \begin{aligned}
        \mathbb{E}(\sum_{t=1}^{T-1}f( \boldsymbol{x}_{t})+F( \boldsymbol{x}_{T}))\ge(1-e^{-\gamma})(T-1)f( \boldsymbol{x}^{*})-(\mathrm{diam}^{2}(\mathcal{C})L_{\gamma}/2+3\sigma_{\gamma}\mathrm{diam}(\mathcal{C})\sqrt{T}/2)
     \end{aligned}
 \end{equation}
According to \cref{thm:2},
\begin{equation}\label{equ:50}
     \begin{aligned}
        \mathbb{E}(\sum_{t=1}^{T-1}f(\boldsymbol{x}_{t})+(1+\log(\tau))(f(\boldsymbol{x}_{T})+c)\ge(1-e^{-\gamma})(T-1)f(\boldsymbol{x}^{*})-(\mathrm{diam}^{2}(\mathcal{C})L_{\gamma}/2+3\sigma_{\gamma}\mathrm{diam}(\mathcal{C})\sqrt{T}/2)
     \end{aligned}
 \end{equation} where $\tau=\max(\frac{1}{\gamma},\frac{r^{2}(\mathcal{X})L}{c})$.
 
In \cref{alg:1}, we set
\begin{equation}
\triangle_{t}=\left\{
    \begin{aligned}
      &1 & t\neq T\\
      &1+\log(\tau) & t=T
    \end{aligned}\right.
\end{equation} and $\triangle=\sum_{t=1}^{T}\triangle_{t}=T+\log(\tau)$.
\begin{equation}\label{equ:52}
     \begin{aligned}
        \mathbb{E}(\sum_{t=1}^{T}\frac{\triangle_{t}}{\triangle}f(\boldsymbol{x}_{t}))\ge(1-e^{-\gamma}-\frac{1+\ln(\tau)}{T+\ln(\tau)})f(\boldsymbol{x}^{*})-\frac{(\mathrm{diam}^{2}(\mathcal{C})L_{\gamma}/2+3\sigma_{\gamma}\mathrm{diam}(\mathcal{C})\sqrt{T})/2+(1+\log(\tau))c}{T+\ln(\tau)}
     \end{aligned}
 \end{equation}
 
 Therefore, when $c=O(1)$, we have
\begin{align*}
        \mathbb{E}(\sum_{t=1}^{T}\frac{\triangle_{t}}{\triangle}f(\boldsymbol{x}_{t}))\ge(1-e^{-\gamma}-O(\frac{1}{T}))f(\boldsymbol{x}^{*})-O(\frac{1}{\sqrt{T}})
     \end{align*}
\end{proof}
% \section{Proof of Theorem 4}\label{Appendix:C}
% \input{sections/Appendix-C}
\section{Proof of Theorem~\ref{thm:5}}\label{Appendix:D}
\begin{proof}
% Based on the proof of \cite{quanrud2015online}, we verify the \cref{thm:5}. Also, 
We denote $\widetilde{\nabla} F_{t}(\boldsymbol{x}_t)=\frac{1-e^{-\gamma}}{\gamma}\widetilde{\nabla}f(z_{t}*\boldsymbol{x}_{t})$ and $\boldsymbol{x}^{*}=\arg\max_{\boldsymbol{x}\in\mathcal{C}}\sum_{t=1}^{T}f_{t}(\boldsymbol{x})$. From the projection, we know that
\begin{equation}\label{equ:b43}
    \begin{aligned}
       \left\|\boldsymbol{x}_{t+1}-\boldsymbol{x}^{*}\right\|&\le\left\|\boldsymbol{y}_{t+1}-\boldsymbol{x}^{*}\right\|=\left\|\boldsymbol{x}_{t}+\eta\sum_{s\in\mathcal{F}_{t}}\widetilde{\nabla} F_{s}(\boldsymbol{x}_{s})-\boldsymbol{x}^{*}\right\|,
    \end{aligned}
\end{equation}where the first inequality from the projection; and the first equality from $\boldsymbol{y}_{t+1}=\boldsymbol{x}_{t}+\eta\sum_{s\in\mathcal{F}_{t}}\frac{1-e^{-\gamma}}{\gamma}\widetilde{\nabla}f_{s}(z_{s}*\boldsymbol{x}_{s})$ in \cref{alg:2}.

We order the set  $\mathcal{F}_{t}=\{s_1,\dots,s_{|\mathcal{F}_{t}|}\}$, where $s_1<s_2<\dots<s_{|\mathcal{F}_{t}|}$ and $|\mathcal{F}_{t}|=\#\{u\in[T]: u+d_{u}-1=t\}$. Moreover, we also denote $\mathcal{F}_{t,m}=\{u\in\mathcal{F}_{t}\ and\ u<m\}$, $\boldsymbol{x}_{t+1,m}=x_{t}+\eta\sum_{s\in\mathcal{F}_{t,m}}\widetilde{\nabla} F_{s}(\boldsymbol{x}_{s})$ and $s_{|\F_{t}|+1}=t+1$. Therefore,
\begin{equation}\label{equ:b44}
    \begin{aligned}
       \left\|\boldsymbol{x}_{t+1,s_{k+1}}-\boldsymbol{x}^{*}\right\|^{2}&=\left\|\boldsymbol{x}_{t+1,s_{k}}+\eta\widetilde{\nabla}F_{s_{k}}(\boldsymbol{x}_{s_{k}})-\boldsymbol{x}^{*}\right\|^{2}\\
       &=\left\|\boldsymbol{x}_{t+1,s_{k}}-\boldsymbol{x}^{*}\right\|^{2}+2\eta\langle \boldsymbol{x}_{t+1,s_{k}}-\boldsymbol{x}^{*},\widetilde{\nabla}F_{s_{k}}(\boldsymbol{x}_{s_{k}})\rangle+\eta^2\left\|\widetilde{\nabla}F_{s_{k}}(\boldsymbol{x}_{s_{k}})\right\|^{2}
       \end{aligned}
\end{equation}
According to \cref{equ:b44}, we have

\begin{equation}\label{equ:b45}
    \begin{aligned}
       &\left\|\boldsymbol{y}_{t+1}-\boldsymbol{x}^{*}\right\|^{2}-\left\|\boldsymbol{x}_{t}-\boldsymbol{x}^{*}\right\|^{2}\\
       = & \sum_{k=1}^{|\mathcal{F}_{t}|}(\left\|\boldsymbol{x}_{t+1,s_{k+1}}-\boldsymbol{x}^{*}\right\|^{2}-\left\|\boldsymbol{x}_{t+1,s_{k}}-\boldsymbol{x}^{*}\right\|^{2})\\
       = & 2\eta\sum_{s\in\mathcal{F}_{t}}\langle \boldsymbol{x}_{t+1,s}-\boldsymbol{x}^{*},\widetilde{\nabla}F_{s}(\boldsymbol{x}_{s})\rangle+\eta^2\sum_{s\in\mathcal{F}_{t}}\left\|\widetilde{\nabla}F_{s}(\boldsymbol{x}_{s})\right\|^{2}\\
       = & 2\eta\sum_{s\in\mathcal{F}_{t}}\langle \boldsymbol{x}_{t+1,s}-\boldsymbol{x}_{s},\widetilde{\nabla}F_{s}(\boldsymbol{x}_{s})\rangle+2\eta\sum_{s\in\mathcal{F}_{t}}\langle \boldsymbol{x}_{s}-\boldsymbol{x}^{*},\widetilde{\nabla}F_{s}(\boldsymbol{x}_{s})\rangle+\eta^2\sum_{s\in\mathcal{F}_{t}}\left\|\widetilde{\nabla}F_{s}(\boldsymbol{x}_{s})\right\|^{2}
       \end{aligned}
\end{equation} where the first equality follows from setting $\boldsymbol{x}_{t+1,|\mathcal{F}_{t}|+1}=\boldsymbol{y}_{t+1}$; the second from \cref{equ:b44}. 

\noindent Therefore,
\begin{equation}\label{equ:b46}
    \begin{aligned}
       &\mathbb{E}(\left\|\boldsymbol{y}_{t+1}-\boldsymbol{x}^{*}\right\|^{2}-\left\|\boldsymbol{x}_{t}-\boldsymbol{x}^{*}\right\|^{2})\\
       = & 2\eta\mathbb{E}\left(\sum_{s\in\mathcal{F}_{t}}\langle \boldsymbol{x}_{t+1,s}-\boldsymbol{x}_{s},\widetilde{\nabla}F_{s}(\boldsymbol{x}_{s})\rangle+\sum_{s\in\mathcal{F}_{t}}\langle \boldsymbol{x}_{s}-\boldsymbol{x}^{*},\mathbb{E}(\widetilde{\nabla}F_{s}(\boldsymbol{x}_{s})|\boldsymbol{x}_{s})\rangle\right)+\eta^2\mathbb{E}(\sum_{s\in\mathcal{F}_{t}}\left\|\widetilde{\nabla}F_{s}(\boldsymbol{x}_{s})\right\|^{2})\\
       = & 2\eta\mathbb{E}\left(\sum_{s\in\mathcal{F}_{t}}\langle \boldsymbol{x}_{t+1,s}-\boldsymbol{x}_{s},\widetilde{\nabla}F_{s}(\boldsymbol{x}_{s})\rangle+\sum_{s\in\mathcal{F}_{t}}\langle \boldsymbol{x}_{s}-\boldsymbol{x}^{*},\nabla F_{s}(\boldsymbol{x}_{s})\rangle\right)+\eta^2\mathbb{E}(\sum_{s\in\mathcal{F}_{t}}\left\|\widetilde{\nabla}F_{s}(\boldsymbol{x}_{s})\right\|^{2})\\
       \le & 2\eta\mathbb{E}\left(\sum_{s\in\mathcal{F}_{t}}\langle \boldsymbol{x}_{t+1,s}-\boldsymbol{x}_{s},\widetilde{\nabla}F_{s}(\boldsymbol{x}_{s})\rangle+\sum_{s\in\mathcal{F}_{t}}\left(f_{s}(\boldsymbol{x}_{s})-(1-e^{-\gamma})f_{s}(\boldsymbol{x}^{*})\right)\right)+\eta^2\mathbb{E}(\sum_{s\in\mathcal{F}_{t}}\left\|\widetilde{\nabla}F_{s}(\boldsymbol{x}_{s})\right\|^{2})
      \end{aligned}
\end{equation}where the first inequality from the definition of non-oblivious function $F$.

\noindent Therefore, we have:
\begin{equation}\label{equ:b47}
    \begin{aligned}
       &2\eta\mathbb{E}\left((1-e^{-\gamma})\sum_{t=1}^{T}f_t(\boldsymbol{x}^{*})-\sum_{t=1}^{T}f_{t}(\boldsymbol{x}_{t})\right)\\
       = & 2\eta\mathbb{E}\left(\sum_{t=1}^{T}\sum_{s\in\mathcal{F}_{t}}\left((1-e^{-\gamma})f_s(\boldsymbol{x}^{*})-f_{s}(\boldsymbol{x}_{s})\right)\right)\\
       \le & \sum_{t=1}^{T}\left(\mathbb{E}(\left\|\boldsymbol{x}_{t}-\boldsymbol{x}^{*}\right\|^{2}-\left\|\boldsymbol{y}_{t+1}-\boldsymbol{x}^{*}\right\|^{2})+2\eta\mathbb{E}(\sum_{s\in\mathcal{F}_{t}}\langle \boldsymbol{x}_{t+1,s}-\boldsymbol{x}_{s},\widetilde{\nabla}F_{s}(\boldsymbol{x}_{s})\rangle)+\eta^2\mathbb{E}(\sum_{s\in\mathcal{F}_{t}}\left\|\widetilde{\nabla}F_{s}(\boldsymbol{x}_{s})\right\|^{2})\right)\\
       \le & \sum_{t=1}^{T}\left(\mathbb{E}(\left\|\boldsymbol{x}_{t}-\boldsymbol{x}^{*}\right\|^{2}-\left\|\boldsymbol{x}_{t+1}-\boldsymbol{x}^{*}\right\|^{2})+2\eta\mathbb{E}(\sum_{s\in\mathcal{F}_{t}}\langle \boldsymbol{x}_{t+1,s}-\boldsymbol{x}_{s},\widetilde{\nabla}F_{s}(\boldsymbol{x}_{s})\rangle)+\eta^2\mathbb{E}(\sum_{s\in\mathcal{F}_{t}}\left\|\widetilde{\nabla}F_{s}(\boldsymbol{x}_{s})\right\|^{2})\right)\\
       \le & \mathrm{diam}^{2}(\mathcal{C})+\sum_{t=1}^{T}\left(2\eta\mathbb{E}(\sum_{s\in\mathcal{F}_{t}}\langle \boldsymbol{x}_{t+1,s}-\boldsymbol{x}_{s},\widetilde{\nabla}F_{s}(\boldsymbol{x}_{s})\rangle)+\eta^2\mathbb{E}(\sum_{s\in\mathcal{F}_{t}}\left\|\widetilde{\nabla}F_{s}(\boldsymbol{x}_{s})\right\|^{2})\right)\\
       \le & \mathrm{diam}^{2}(\mathcal{C})+\eta^{2}\max_{t\in[T]}(\left\|\widetilde{\nabla}F_{t}(\boldsymbol{x}_{t})\right\|^{2})\sum_{t=1}^{T}|\mathcal{F}_{t}|+2\eta\sum_{t=1}^{T}\left(\mathbb{E}\left(\sum_{s\in\mathcal{F}_{t}}\langle \boldsymbol{x}_{t+1,s}-\boldsymbol{x}_{s},\widetilde{\nabla}F_{s}(\boldsymbol{x}_{s})\rangle\right)\right)
    \end{aligned}
\end{equation}
\noindent For the final part in \cref{equ:b47}, 
\begin{equation}\label{equ:b48}
    \begin{aligned}
       &\langle \boldsymbol{x}_{t+1,s}-\boldsymbol{x}_{s},\widetilde{\nabla}F_{s}(\boldsymbol{x}_{s})\rangle\\
       \le & \left\|\widetilde{\nabla}F_{s}(\boldsymbol{x}_{s})\right\|\left\|\boldsymbol{x}_{t+1,s}-\boldsymbol{x}_{s}\right\|\\
       \le & \left\|\widetilde{\nabla}F_{s}(\boldsymbol{x}_{s})\right\|(\left\|\boldsymbol{x}_{t+1,s}-\boldsymbol{x}_{t}\right\|+\left\|\boldsymbol{x}_{t}-\boldsymbol{x}_{s}\right\|)\\
       \le & \left\|\widetilde{\nabla}F_{s}(\boldsymbol{x}_{s})\right\|(\left\|\boldsymbol{x}_{t+1,s}-\boldsymbol{x}_{t}\right\|+\sum_{m=s}^{t-1}\left\|\boldsymbol{y}_{m+1}-\boldsymbol{x}_{m}\right\|)\\
       \le & \max_{t\in[T]}(\left\|\widetilde{\nabla}F_{t}(\boldsymbol{x}_{t})\right\|^{2})\eta(|\mathcal{F}_{t,s}|+\sum_{m=s}^{t-1}|\mathcal{F}_{m}|)
    \end{aligned}
\end{equation} where the third inequality follows from $\left\|\boldsymbol{x}_{t}-\boldsymbol{x}_{s}\right\|\le\left\|\boldsymbol{y}_{t}-\boldsymbol{x}_{s}\right\|\le\left\|\boldsymbol{y}_{t}-\boldsymbol{x}_{t-1}\right\|+\left\|\boldsymbol{x}_{t-1}-\boldsymbol{x}_{s}\right\|\le\dots\le\sum_{m=s}^{t-1}\left\|\boldsymbol{y}_{m+1}-\boldsymbol{x}_{m}\right\|$.

\noindent Finally, we have
\begin{equation}\label{equ:b49}
    \begin{aligned}
       &\mathbb{E}\left((1-e^{-\gamma})\sum_{t=1}^{T}f_t(\boldsymbol{x}^{*})-\sum_{t=1}^{T}f_{t}(\boldsymbol{x}_{t})\right)\\
       \le & \frac{\mathrm{diam}^{2}(\mathcal{C})}{2\eta}+\max_{t\in[T]}(\left\|\widetilde{\nabla}F_{t}(\boldsymbol{x}_{t})\right\|^{2})(\frac{\eta}{2}\sum_{t=1}^{T}|\mathcal{F}_{t}|+\eta\sum_{t=1}^{T}\sum_{s\in\mathcal{F}_{t}}(|\mathcal{F}_{t,s}|+\sum_{m=s}^{t-1}|\mathcal{F}_{m}|))
    \end{aligned}
\end{equation}

Firstly, $\sum_{t=1}^{T}|\mathcal{F}_{t}|\le T$. Next, we investigate the $|\mathcal{F}_{t,s}|+\sum_{m=s}^{t-1}|\mathcal{F}_{m}|$ when $s\in\mathcal{F}_{t}$. 

When $s\in\mathcal{F}_{t}$, i.e., $s+d_{s}-1=t$, for any $q\in(\mathcal{F}_{t,s}\bigcup(\cup_{m=s}^{t-1}\mathcal{F}_{m}))$, if $s+1\le q\le t-1$, the feedback of round $q$ must be delivered before the round $t$, namely, $q+d_{q}-1\le t-1$. Moreover, if $q\le s-1$, the feedback of round $q$ could be delivered between round $s$ and round $t$. Therefore, 
\begin{equation}\label{equ:b50}
    \begin{aligned}
       |\mathcal{F}_{t,s}|+\sum_{m=s}^{t-1}|\mathcal{F}_{m}|= & |\{i| s+1\le i\le t-1,\ and \ i+d_{i}-1\le t-1 \}|\\&+|\{i| 1\le i\le s-1,\ and \ s\le i+d_{i}-1\le t \}|.
    \end{aligned}
\end{equation} 

\noindent When $s\in\mathcal{F}_{t}$, we can derive that $|\{i| s+1\le i\le t-1,\ and\ i+d_{i}-1\le t-1 \}|\le t-s-1\le d_{s}$. Thus, $\sum_{t=1}^{T}\sum_{s\in\mathcal{F}_{t}}|\{i| s+1\le i\le t-1,\ and\ i+d_{i}-1\le t-1 \}|\le\sum_{i=1}^{T}d_{i}=D$.

Next, for each $b\in\{i| 1\le i\le s-1,\ and\ s\le i+d_{i}-1\le t \}$, we have $b\le s\le b+d_{b}-1\le s+d_{s}-1$ so that $\sum_{t=1}^{T}\sum_{s\in \mathcal{F}_{t}}|\{i| 1\le i\le s-1,\ and\ s \le i+d_{i}-1\le t \}|\le\sum_{i=1}^{T}|\{s|\  i<s\le i+d_{i}-1\le s+d_{s}-1\}|\le\sum_{i=1}^{T} d_{i}$. 

Hence,
\begin{equation}\label{equ:51}
    \begin{aligned}
       &\mathbb{E}((1-e^{-\gamma})\sum_{t=1}^{T}f_t(\boldsymbol{x}^{*})-\sum_{t=1}^{T}f_{t}(\boldsymbol{x}_{t}))\\
       \le & \frac{\mathrm{diam}^{2}(\mathcal{C})}{2\eta}+\max_{t\in[T]}(\left\|\widetilde{\nabla}F_{t}(\boldsymbol{x}_{t})\right\|^{2})(\frac{\eta}{2}T+2\eta D)\\
       \le & \frac{\mathrm{diam}^{2}(\mathcal{C})}{2\eta}+\max_{t\in[T]}(\left\|\widetilde{\nabla}F_{t}(\boldsymbol{x}_{t})\right\|^{2})3\eta D\\
       \le & O(\sqrt{D})
    \end{aligned}
\end{equation} where the final equality from $\eta=\frac{\mathrm{diam}(\mathcal{C})}{\max_{t\in[T]}(\left\|\widetilde{\nabla}F_{t}(\boldsymbol{x}_{t})\right\|)\sqrt{D}}$.
\end{proof}

\end{document}